\documentclass{article} 
\usepackage{iclr2025_conference,times}

\usepackage{amsmath,amsfonts,bm}

\def\eqref#1{equation~\ref{#1}}

\def\1{\bm{1}}

\DeclareMathAlphabet{\mathsfit}{\encodingdefault}{\sfdefault}{m}{sl}
\SetMathAlphabet{\mathsfit}{bold}{\encodingdefault}{\sfdefault}{bx}{n}

\usepackage{hyperref}
\usepackage{url}
\usepackage{graphicx}
\usepackage{graphics}
\usepackage{cases}
\usepackage{multirow}
\usepackage{algorithm}  
\usepackage{amsmath}
\usepackage{algorithmic}  
\usepackage{lineno}
\usepackage{amsfonts}
\usepackage{amsmath} 
\usepackage{amsthm}  
\usepackage{graphicx}
\usepackage{bm}
\usepackage{algorithmic}
\usepackage{textcomp}
\usepackage{xcolor}
\usepackage{algorithm}  

\usepackage{booktabs}
\usepackage{footnote}
\usepackage{CJK}
\usepackage{indentfirst}
\usepackage{fancyhdr}

\title{The Language of Time: A Language Model Perspective on Time Series Foundation Models}

\author{
Yi Xie\textsuperscript{1,2}, 
Yun Xiong\textsuperscript{1,2}, 
Zejian Shi\textsuperscript{3}, 
Hao Niu\textsuperscript{1,2}, 
Zhengfu Liu\textsuperscript{4} \\
\textsuperscript{1}College of Computer Science and Artificial Intelligence, Fudan University, Shanghai, China \\
\textsuperscript{2}Shanghai Key Laboratory of Data Science, Shanghai, China \\
\textsuperscript{3}ZCTech, Hangzhou, China \\
\textsuperscript{4}School of Mathematics and Statistics, Beijing Institute of Technology, Beijing, China \\
\texttt{\textsuperscript{1,2}\{yixie18, yunx, hniu\}@fudan.edu.cn}, \\ \texttt{\textsuperscript{3}shizejian@zctech-ai.com, \textsuperscript{4}3120235975@bit.edu.cn}
}

\newtheorem{proposition}{Proposition}[section]

\newtheorem{assumption}{Assumption}[section]
\newtheorem{lemma}{Lemma}[section]
\newtheorem{remark}{Remark}[section]
\newtheorem{theorem}{Theorem}[section]
\newtheorem{corollary}{Corollary}[section]
\newtheorem{definition}{Definition}[section]
\iclrfinalcopy 
\fancypagestyle{clean}{%
  \fancyhf{} 
  \fancyfoot[C]{\thepage} 
  \renewcommand{\headrulewidth}{0pt} 
  \renewcommand{\footrulewidth}{0pt} 
}
\begin{document}

\maketitle
\thispagestyle{clean}

\begin{abstract}
With the rise of large language models, the paradigm of training foundation models with massive parameter counts on vast datasets has been adopted in multiple domains to achieve remarkable success. Time series foundation models represent a significant extension of this paradigm, demonstrating exceptional expressive power, generalization, and cross-domain transferability. However, this gives rise to a fundamental paradox: time series data reflect distinct dynamical systems, making cross-domain transfer intuitively implausible, yet this is contradicted by the models' empirical success.
To resolve this paradox, this paper investigates, from both theoretical and experimental perspectives, the representation learning mechanisms and generalization capabilities of patch-based time series foundation models. We argue that such models are not merely applying a new architecture but are fundamentally generalizing the representation paradigm of language models by \textbf{extending deterministic vector-based representations to latent probabilistic distributional forms}. Our theoretical analysis supports this framework by demonstrating that continuous time-series patches can be faithfully quantized into a discrete vocabulary whose key statistical properties are highly consistent with those of natural language.
This generalization allows time series models to inherit the robust representation and transfer abilities of large language models, thereby explaining their superior performance in temporal tasks. Ultimately, our work provides a rigorous theoretical cornerstone for understanding, evaluating, and improving the safety and reliability of large-scale time series foundation models.
\end{abstract}

\section{Introduction}
Time series data chronicle the evolution of complex systems through sequences of numerical observations sampled at uniform intervals, offering a quantitative fingerprint of their dynamics \cite{ts_def1}. The inherent characteristics of these data—most notably temporal dependencies, underlying trends, and seasonal cycles—make them invaluable for a vast array of applications, including traffic flow forecasting, logistics optimization, climate change analysis, and human mobility modeling \cite{traffic,traffic2,remote,humanmob}.

Mirroring the paradigm shift driven by large language models, time-series foundation models have recently achieved remarkable success through large-scale pre-training and fine-tuning \cite{choronos,moment,timesfm}. By pre-training on massive and diverse corpora, often encompassing billions of temporal data points, these models can be adapted using zero-shot or few-shot strategies. This has led to substantial improvements in forecasting accuracy, robust cross-domain generalization, and impressive performance even with limited data \cite{survey,ICL-TS}.

Yet, this empirical success stands in stark contrast to a growing body of critical analysis questioning the internal mechanisms, the true efficacy of domain transfer, and the fundamental in-context learning capabilities of these models \cite{argue1,argue2,argue3,argue4}. The core challenge is clear: each time series represents a unique system with its own distinct temporal patterns. Consequently, transferring a model between disparate domains—for instance, from energe consumption to climate science—inevitably incurs a significant distributional shift. This chasm between what these models can do and why they should work raises fundamental questions about their safety, reliability, and theoretical underpinnings. Are their impressive feats merely an over-fitting coincidence of massive data and computation, or can they be anchored in a rigorous theoretical foundation?

In this paper, we bridge this gap from a theoretical standpoint. We argue that a patch-embedding-based time series foundation model can be formally understood as a generalization of a large language model—one that operates not on discrete tokens, but on token distributions.

Our core argument lies in re-conceptualizing the fundamental unit of input: while language models process discrete tokens (words), time-series models should treat \textit{patches}---short temporal segments---as their basic unit. Unlike words that map to isolated points in latent space, time-series patches correspond to families of patterns, or recurring temporal \textit{motifs} \cite{motif1,motif2}. For instance, a \textit{gradual decrease} motif may manifest as variants with differing slopes or noise levels (see Figure~\ref{fig:patches}), which, despite numerical differences, belong to the same conceptual family. Consequently, their embeddings should form distributions in latent space rather than single points (see Figure~\ref{fig:embeddings}). Notably, motifs exhibit significant co-occurrence relationships: two peaks necessarily imply an intervening trough (and vice versa); sharp upward trends are likely followed by decay or correction phases, forming \textit{steep rise-sudden drop} co-occurrence patterns; or certain motifs appear in pairs to constitute complete cycles, creating \textit{rising edge-falling edge} binding relationships. Beyond these intuitive examples, there exist numerous other motif relationships that defy verbal description yet exist in practice---a phenomenon remarkably analogous to lexical co-occurrence in language models.

\begin{figure}[htbp]
  \centering
  \includegraphics[width=\columnwidth]{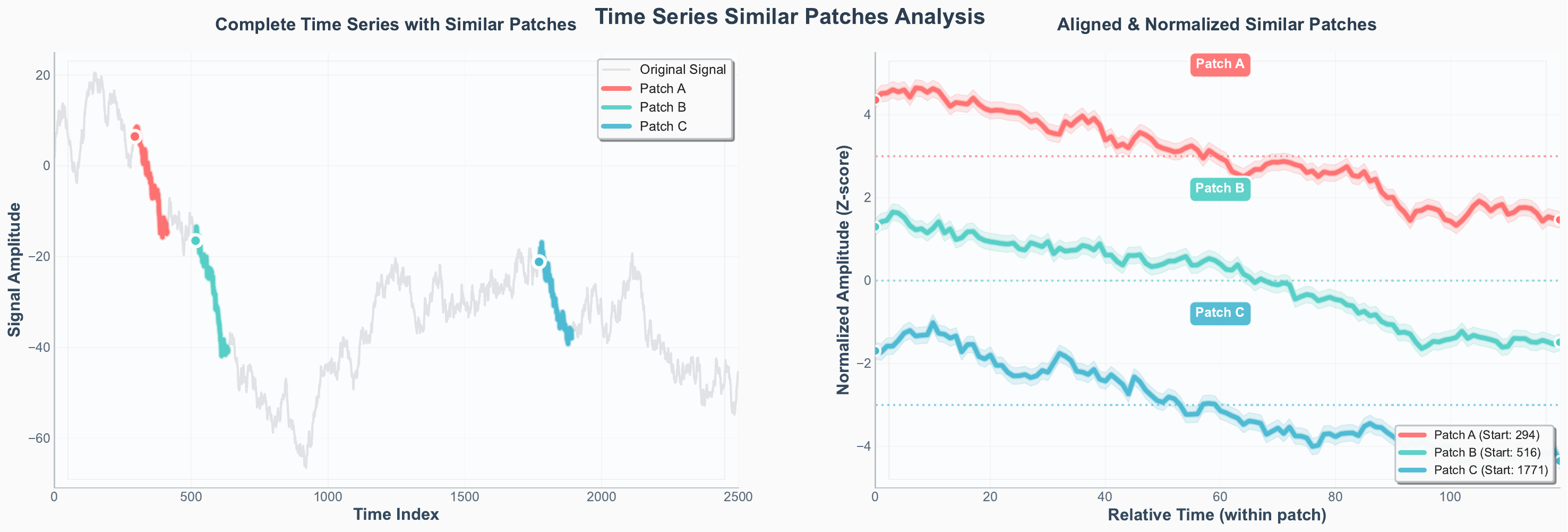} 
  \caption{%
Visualization of similar time-series segments and their alignment. 
\textbf{Left}: Three highlighted patches (Patch A/B/C) in the full signal share an identical trend shape despite differing amplitudes. 
\textbf{Right}: After Z-score normalization and temporal alignment, the curves almost overlap, indicating that the segments belong to the same latent temporal motif.%
} 
  \label{fig:patches}   
\end{figure}

\begin{figure}[htbp]
  \centering
  \includegraphics[width=\columnwidth]{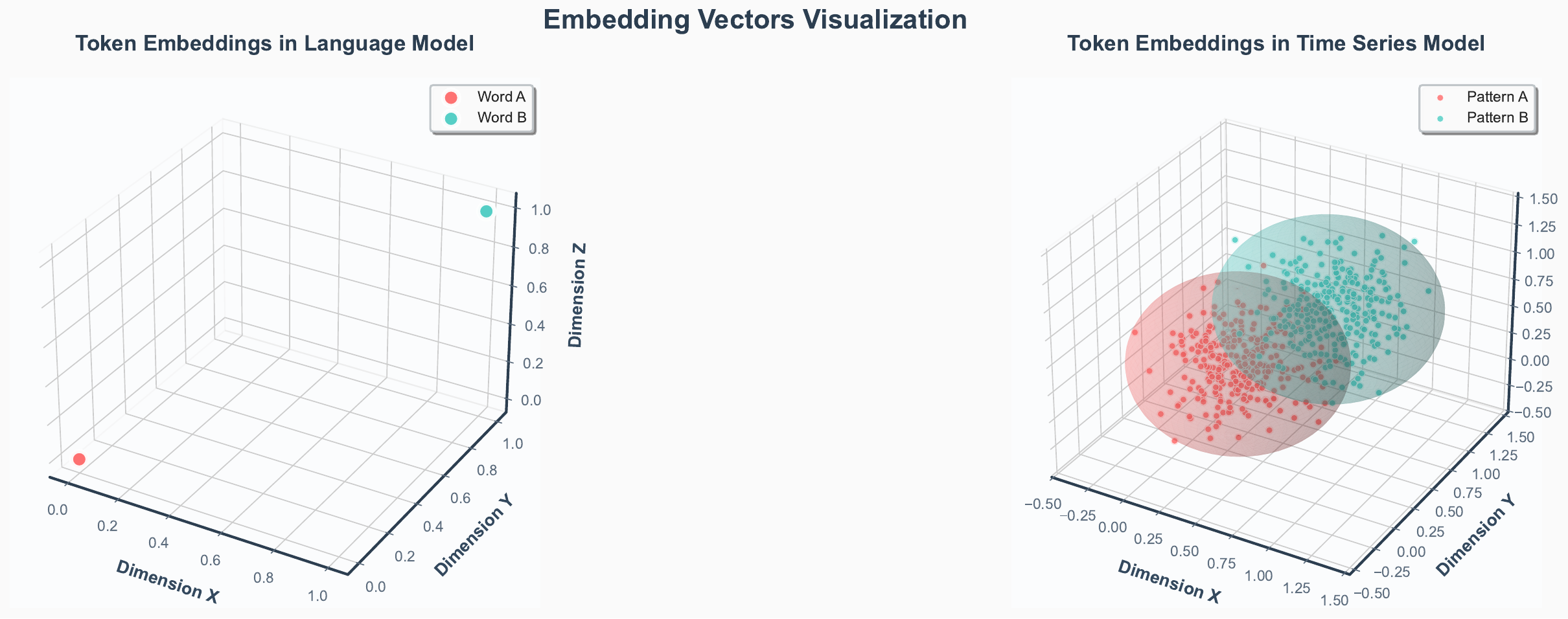} 
  \caption{%
Distributed embeddings of language tokens versus time-series patches. 
\textbf{Left}: In a language model, token embeddings appear as discrete, sparsely located singleton points. 
\textbf{Right}: In a time-series model, patch embeddings form probability clouds with finite thickness; patches of the same motif (Pattern A/B) cluster into separable yet internally continuous regions, illustrating the concept of a “distributional token.”%
}
  \label{fig:embeddings}   
\end{figure}

Furthermore, these pattern families are not strictly partitioned or mutually exclusive; a single patch might simultaneously exhibit characteristics of multiple motifs, leading to overlapping latent-vector distributions. Crucially, we posit that this extension from point-wise representations to distributional ones is what allows the model to inherit the expressive power and powerful generalization capabilities of LLMs. It is this ability to learn an abstract vocabulary of continuous temporal patterns, rather than memorizing specific numerical sequences, that provides a rigorous theoretical justification for the observed success of patch-embedding-based time-series foundation models.

To substantiate our proposed "distributional token" hypothesis, this paper unfolds an in-depth investigation from both empirical and theoretical perspectives. We conduct a series of carefully designed empirical studies aimed at validating the key assumptions of our theoretical framework and demonstrating their manifestation in real-world data. Concurrently, through a set of rigorous and hierarchical theoretical derivations, we establish a solid mathematical foundation for this novel perspective, transforming intuitive insights into provable conclusions.

Our main contributions can be summarized as follows:
\begin{itemize}
    \item \textbf{Empirical Discovery of "Quasi-Linguistic" Properties of Time}: We are the first to show, through large-scale empirical analysis, that after patches extracted from diverse time-series datasets are quantized into tokens, their frequency distribution strikingly follows a Zipf-like law. This provides strong statistical evidence for the concept of a "language of time." Furthermore, our experiments validate the natural denoising effect of the patching operation across various tasks.
    \item \textbf{Construction of a Hierarchical Theoretical Framework:} We build a complete theoretical analysis framework to support our claims. This framework:
    \begin{itemize}
        \item Establishes the fidelity of representing continuous patches with discrete tokens, starting from covering number theory.
        \item Proves that the patch-based hypothesis space has non-decreasing representational capacity, using principles from learning theory to ensure model expressiveness.
        \item Leverages the Information Bottleneck principle to reveal how patching acts as an effective compression scheme that filters out noise while preserving task-relevant information.
        \item Finally, by proving the key property of dependence preservation, it connects our framework with generalization theory for dependent sequences, thereby providing guarantees for the model's generalization ability.
    \end{itemize}
    \item \textbf{A Bridge Between Theory and Practice}: We clearly elucidate the link between theory and practice by demonstrating why patch-based methods naturally achieve pattern abstraction. This explains the fundamental reason for the successful cross-domain transfer of time-series foundation models: they learn a universal and composable vocabulary of temporal dynamic "motifs," rather than memorizing domain-specific numerical sequences.
\end{itemize}

\section{Empirical Validation}

Our central hypothesis posits that time series data harbors a deep statistical structure analogous to that of natural language. To empirically validate this claim, we conducted a large-scale experiment to investigate whether non-trivial statistical laws emerge when continuous temporal dynamics are symbolized into a discrete vocabulary. This was achieved by applying K-Means clustering to a vast dataset of 38k time series patches from diverse domains, thereby quantizing continuous patterns into a vocabulary of "Temporal Words." The objective was to test whether the usage patterns of this data-driven vocabulary conform to the same statistical principles that govern human language.

\subsection{The Vocabulary of Time Series}
\subsubsection{Vocabulary Construction}

Our central hypothesis posits that time series data harbors a deep statistical structure analogous to that of natural language. To empirically validate this claim, the primary task is to transform the continuous, high-dimensional time-series signals into a discrete, analyzable representation composed of symbols. The core of this transformation lies in constructing a "Vocabulary of Time Series".  number of ``tokens'' or ``words'' in the time series.

\begin{figure}[htbp]
  \centering
  \includegraphics[width=\columnwidth]{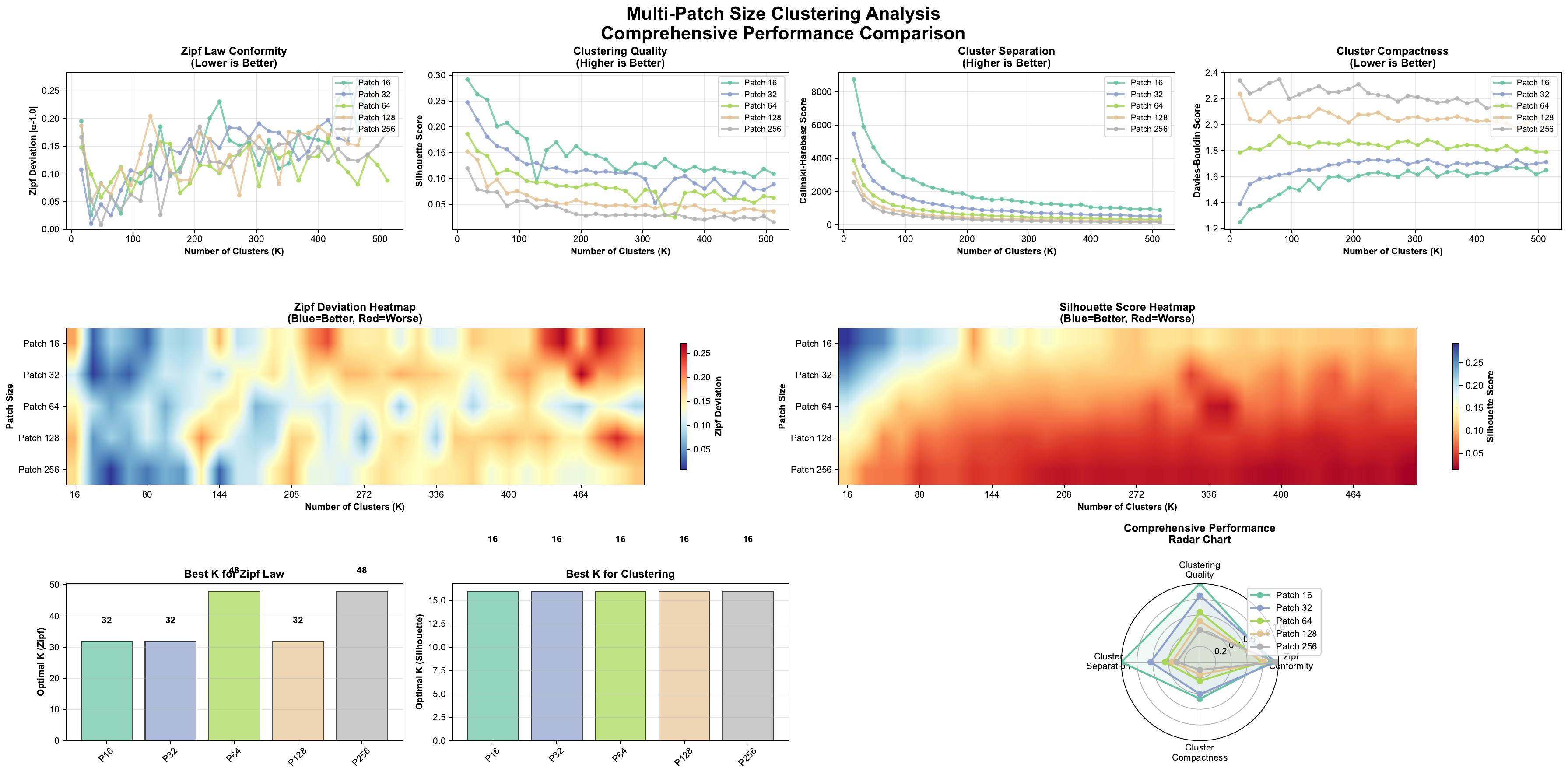} 
  \caption{%
    \textbf{Comprehensive analysis of clustering performance reveals a core trade-off governed by patch size ($P$).} 
    The visualizations (line plots, heatmaps, and radar chart) collectively demonstrate that smaller patches (e.g., $P=16$) excel at forming structurally distinct clusters (high Silhouette Score), whereas larger patches ($P \ge 64$) create vocabularies with greater linguistic plausibility (stronger conformity to Zipf's Law). 
    This shows that patch size is a fundamental design choice, balancing the structural clarity of ``atomic'' patterns against the semantic richness of a more language-like temporal vocabulary.%
  }
  \label{fig:comprehensive_analysis}
\end{figure}

Nevertheless, this vocabulary is not formed from predefined functions or rules but is generated entirely in a data-driven manner. We envision that complex dynamic phenomena arise from the composition of a finite, reusable set of fundamental patterns or "Temporal Motifs". Therefore, the construction process for this vocabulary aims to automatically discover and distill a representative set of such pattern prototypes from vast and diverse time-series data.

Specifically, our construction pipeline is as follows:

\begin{enumerate}
    \item \textbf{Patching}: We segment the raw time series into continuous segments of length $P$ with a stride of $S$, referred to as "patches". Each patch serves as the fundamental unit for our analysis, encapsulating a segment of local dynamic information.

    \item \textbf{Quantization}: To map these continuous, high-dimensional patch vectors into a finite, discrete symbol space, we employ the technique of Vector Quantization. In this study, we select the \textbf{K-Means clustering algorithm} as our core quantization tool. This choice is predicated on two principal advantages:
    (1) \textbf{Intuitive Prototype Discovery}: The K-Means algorithm partitions the data by finding $K$ "centroids". Each centroid is itself a vector of the same dimension as the input patches and can be intuitively interpreted as a standard, clean "pattern prototype".
     (2) \textbf{Computational Efficiency}: We utilize its Mini-Batch variant, ensuring that the method can be efficiently scaled to process massive datasets, a critical requirement for training foundation models.
\end{enumerate}

By executing K-Means clustering on the entire dataset of 38k patch samples, we obtain a "codebook" or vocabulary $\mathcal{C}$ composed of $K$ centroids. Each entry in this vocabulary---that is, each centroid---constitutes a "Temporal Word", representing a fundamental dynamic pattern learned from the data. Subsequently, any given patch from the original dataset can be mapped to a specific index (ID) in this vocabulary by identifying its nearest centroid in Euclidean space.

Ultimately, a continuous time series is successfully transformed into an integer sequence of discrete "token" IDs. This symbolic representation not only significantly compresses the data and filters out noise but, more critically, lays the groundwork for our subsequent statistical analysis. In the following sections, we will conduct an in-depth analysis of the frequency of use of this generated "vocabulary of time" to test whether it truly exhibits the quasi-linguistic properties we hypothesize.

Following our experimental setup, we conducted a comprehensive analysis to evaluate the properties of the generated vocabularies. The results, summarized in the "Multi-Patch Size Clustering Analysis" figure, reveal critical insights into the relationship between tokenization parameters and vocabulary quality. Our analysis of these results reveals a fundamental trade-off between the structural quality of the learned vocabulary and its statistical resemblance to natural language.

The experimental results present several clear trends:
\begin{itemize}
    \item \textbf{Structural Fidelity}: Metrics for clustering quality consistently favor smaller patch sizes. `Patch 16` achieves the highest Silhouette Score and the lowest, or best, Davies-Bouldin Score. As patch size increases from 16 to 256, the maximum achievable Silhouette Score monotonically decreases. For all patch sizes, the optimal vocabulary size for maximizing clustering quality is consistently $K=16$.
    \item \textbf{Linguistic Plausibility}: The conformity to Zipf's Law shows a more complex relationship. Larger patch sizes, such as $P=64$ and $P=256$, demonstrate the ability to achieve the lowest (best) deviation from an ideal Zipfian distribution. The optimal K for achieving this is higher than for clustering, typically falling at $K=32$ or $K=48$ for different patch sizes.
\end{itemize}

The opposing trends in these metrics point to a foundational trade-off in designing time-series vocabularies.
\begin{enumerate}
    \item \textbf{Small patches ($P=16$) excel at creating a vocabulary of high structural fidelity}. These short segments represent simple, "atomic" patterns that are less varied and lower-dimensional. This makes it easier for the K-Means algorithm to partition them into well-defined, compact, and clearly separated clusters, resulting in superior scores on clustering metrics.

    \item \textbf{Large patches ($P \ge 64$) are more adept at forming a vocabulary with high linguistic plausibility}. These longer segments can capture more complete, semantically rich "temporal motifs." A vocabulary composed of these more complex patterns is more diverse and better mirrors the structure of natural language, where a vast number of specific, low-frequency words create a characteristic "long-tail" distribution that conforms to Zipf's Law.
\end{enumerate}

A crucial observation from our analysis is that the cluster compactness, as measured by the Davies-Bouldin Score, is suboptimal across all tested configurations. The scores remain relatively high (where lower indicates better compactness), suggesting that the identified clusters are inherently diffuse rather than tightly concentrated.
This lack of high compactness is not interpreted as a failure of the clustering algorithm but rather as a reflection of an intrinsic property of the data itself. It indicates that the boundaries between different temporal motifs are not sharply defined. This observation aligns with our intuition that a single time-series patch is rarely a pure exemplar of one motif. For instance, a segment may exhibit a primary "upward trend" while also containing secondary "high-frequency volatility." Consequently, its vector representation would naturally lie in the overlapping boundary regions between distinct clusters, reducing the measured compactness of both.
Ultimately, this empirical finding lends strong support to our "distributional token" hypothesis. The observed looseness of the clusters can be seen as the physical manifestation of an underlying probabilistic representation, where each patch is not a discrete point but rather a distribution that may span multiple conceptual motifs. 

The radar chart provides a holistic visualization of this trade-off. It clearly shows that smaller patch sizes (e.g., `Patch 16`) dominate the axes related to clustering quality, while larger patch sizes are more competitive on the axis for Zipf Conformity. Given this finding, we will proceed with the subsequent experiments using $P=16$ and $K=32$ by default. However, it is important to note that this does not imply the chosen parameters are optimal; after all, unlike language models, we do not have a well-defined or fixed vocabulary for time series.


The choice of patch size is not a simple optimization problem but a fundamental design decision that dictates the nature of the learned vocabulary. A model designer must choose between a vocabulary of simple, structurally clear "atomic" patterns (achieved with small patches) or a vocabulary of complex, semantically rich "motifs" that exhibit more language-like statistical properties (achieved with larger patches). This choice directly impacts the subsequent learning paradigm of any foundation model built upon this tokenization scheme.

\subsubsection{Vocabulary Statistics}
\begin{figure}[htbp]
    \centering
    \includegraphics[width=\columnwidth]{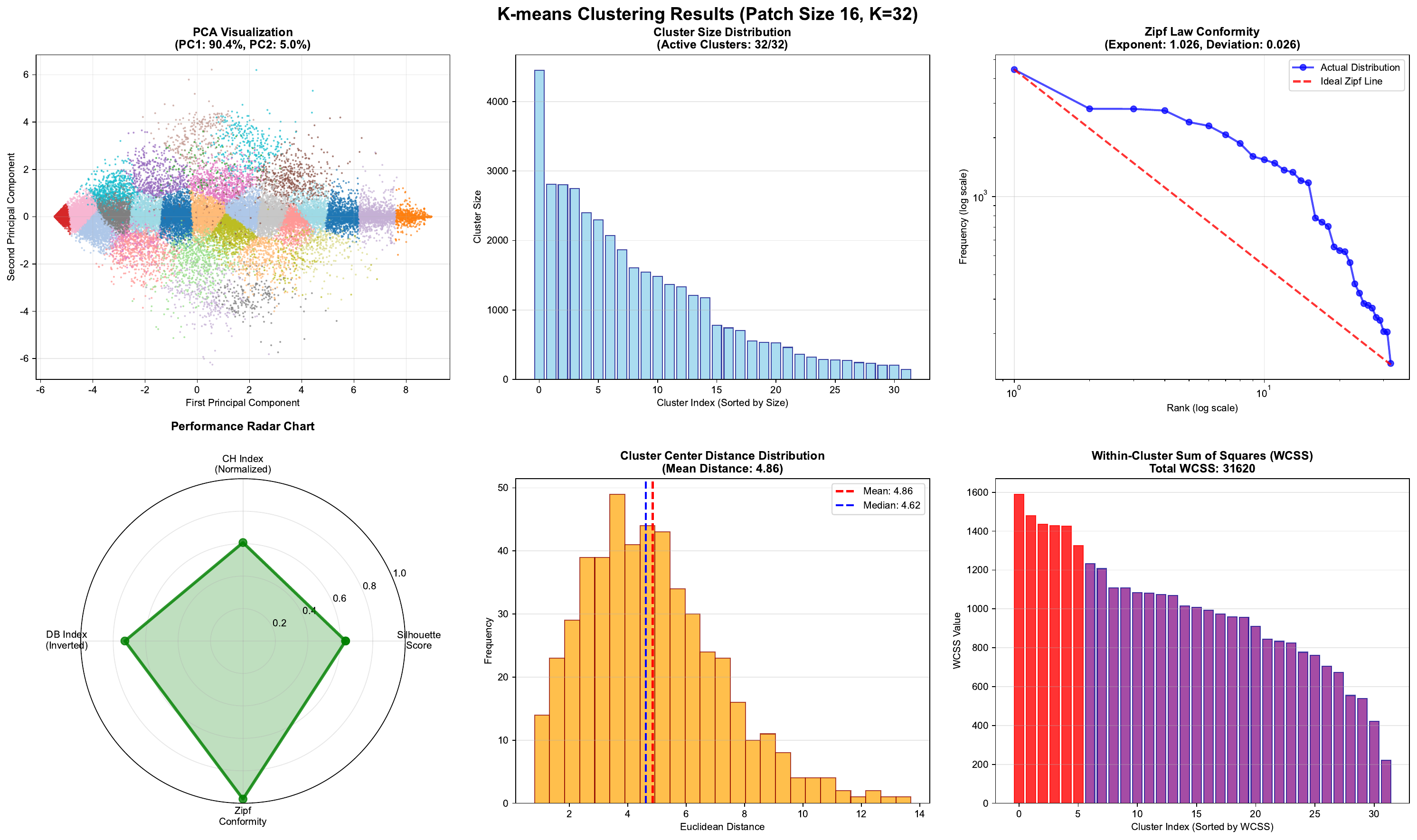}
    \caption{
        \textbf{Comprehensive Statistical Analysis of the Temporal Vocabulary (P=16, K=32).}
        This figure presents a multi-faceted analysis of a "temporal vocabulary" constructed by applying K-Means clustering (K=32) to time series patches of length 16.
        \textbf{(Top-Left)} The \textbf{PCA visualization} reveals that patches form distinct yet overlapping clusters, visually representing the concept of separable but continuous temporal motifs.
        \textbf{(Top-Center \& Top-Right)} The cluster size distribution exhibits a classic long-tail structure. When plotted on a log-log scale, this distribution demonstrates a striking conformity to \textbf{Zipf's Law} (Deviation = 0.026), providing strong quantitative evidence for the quasi-linguistic nature of the data.
        \textbf{(Bottom-Left)} The \textbf{performance radar chart} offers a holistic assessment, showing a strong balance between traditional clustering quality metrics (e.g., Silhouette Score) and the vocabulary's linguistic plausibility (Zipf Conformity).
        \textbf{(Bottom-Center \& Bottom-Right)} The distribution of inter-cluster distances confirms a diverse vocabulary of patterns. Concurrently, the Within-Cluster Sum of Squares (WCSS) highlights the significant internal variance of certain motifs, supporting the \textbf{"distributional token"} hypothesis where each token represents a family of related patterns rather than a single point.
        Collectively, these analyses provide a detailed statistical portrait, confirming that for the P=16, K=32 configuration, the tokenized time series data exhibits a robust, language-like structure.
    }
    \label{fig:statistical_analysis_p16_k32} 
\end{figure}

Figure \ref{fig:statistical_analysis_p16_k32} provides a comprehensive statistical analysis of the temporal vocabulary generated by applying K-Means clustering with parameters set to a patch size of $P=16$ and a vocabulary size of $K=32$. The results offer compelling, multi-faceted evidence for our central hypothesis: that tokenized time series data exhibits a robust, language-like structure.

\paragraph{Cluster Structure and Separability.}
The PCA visualization (top-left panel) reveals the geometric distribution of the patch embeddings after being projected onto their first two principal components. The clusters, denoted by distinct colors, form visually coherent groups that are partially separable, indicating that the K-Means algorithm successfully identified meaningful, recurring patterns. However, the significant overlap between clusters provides initial support for our ``distributional token'' hypothesis, suggesting that temporal motifs are not discrete, isolated points but rather continuous regions in the latent space.

\paragraph{Zipfian Frequency Distribution.}
The most striking finding is the vocabulary's adherence to Zipf's Law. The cluster size distribution (top-center panel) clearly shows a long-tail characteristic, where a few ``temporal words'' are exceedingly common, while the vast majority are rare. This observation is rigorously quantified in the log-log rank-frequency plot (top-right panel). The empirical data points (blue line) align remarkably well with the ideal Zipfian distribution (red dashed line), yielding a Zipf exponent of $1.025$ with a minimal deviation of $0.026$. This strong power-law signature is a hallmark of natural language and provides powerful empirical validation that complex temporal dynamics are composed from a vocabulary of reusable motifs governed by language-like statistical principles.

\paragraph{Holistic Performance and Vocabulary Diversity.}
The performance radar chart (bottom-left) offers a synthesized view of the vocabulary's quality, demonstrating a strong balance between structural fidelity (as measured by the Silhouette, CH, and DB scores) and linguistic plausibility (Zipf Conformity). This indicates that the chosen parameters produce a vocabulary that is both well-structured and statistically sound. Furthermore, the analysis of inter-cluster distances (bottom-center) shows a wide distribution, confirming that the learned vocabulary is diverse, consisting of distinct and well-differentiated temporal patterns.

\paragraph{Evidence for Distributional Tokens.}
Finally, the Within-Cluster Sum of Squares (WCSS) plot (bottom-right) provides further evidence for the distributional nature of temporal tokens. The high WCSS values for several clusters (highlighted in red) are not indicative of poor clustering but rather reflect the high intrinsic variance of those specific motifs. This suggests that a single cluster centroid represents a family of similar, but not identical, temporal patterns (e.g., ``a sharp rise'' with varying slopes and noise levels). This observation reinforces the idea that a token is better understood as a probability distribution over a region of the latent space, rather than a single point vector.

In summary, the collective results presented in Figure \ref{fig:statistical_analysis_p16_k32} establish a solid empirical foundation for viewing time series through a linguistic lens. The discovery of a robust, Zipf-like statistical structure, combined with evidence for distributional representations, provides a fundamental justification for the success of applying large language model paradigms to the time series domain.

\subsection{The Quasi-Linguistic Properties of Time Series}
\label{sec:properties}

\begin{figure}[htbp]
    \centering
    \includegraphics[width=\columnwidth]{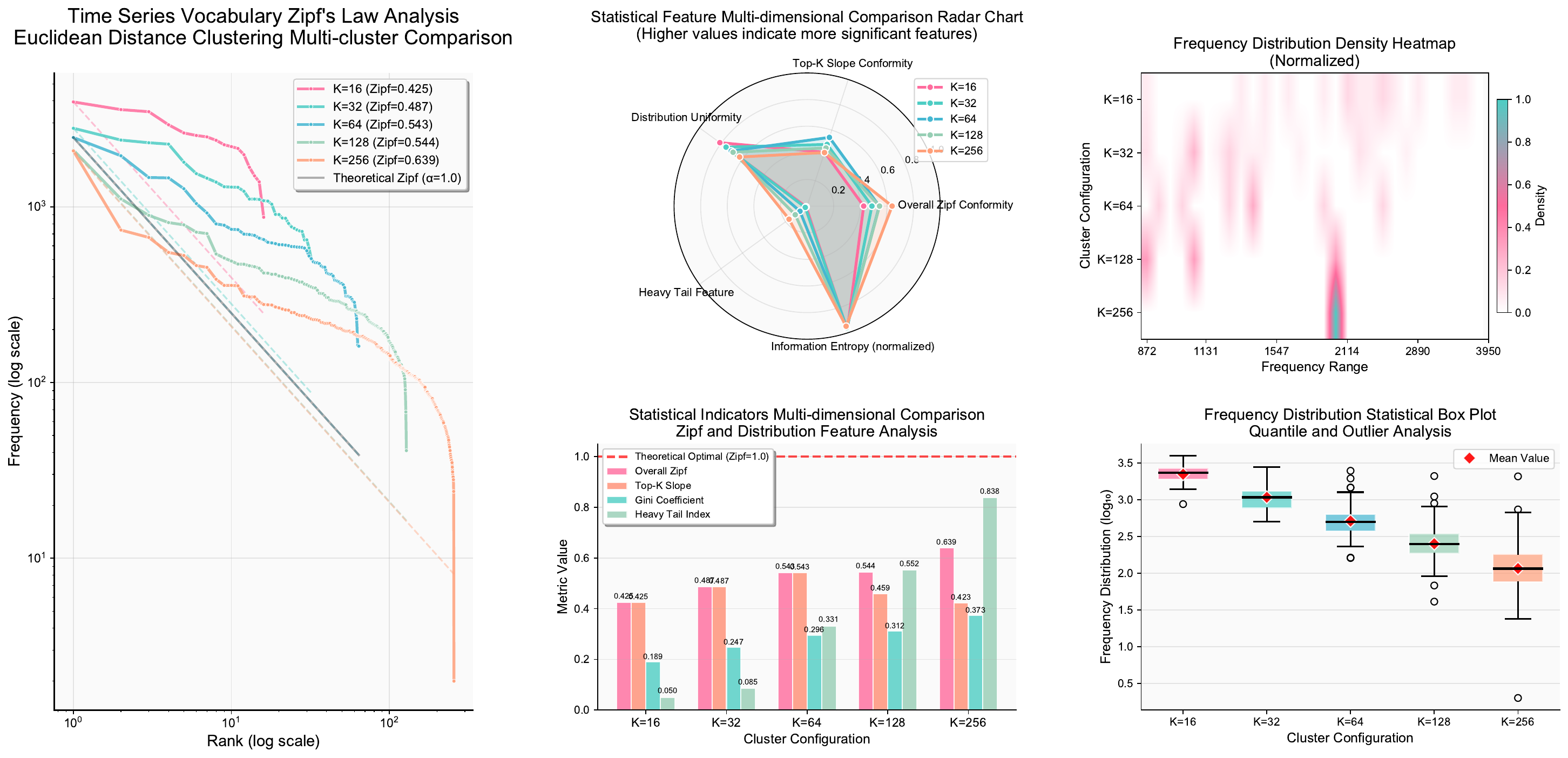}
    \caption{
        \textbf{Statistical Analysis of the Time Series Vocabulary.}
        This figure presents a multi-dimensional analysis of the frequency distribution of "tokens" derived from K-Means clustering on 38,000 time series patches (K=16 to 512).
        \textbf{(Top-left)} The log-log rank-frequency plot reveals a clear Zipf-like power-law distribution across all K values.
        \textbf{(Top-right \& Bottom-right)} The heatmap and boxplots illustrate the distribution's long-tail structure and its dynamic adaptation to varying K.
        \textbf{(Middle)} Quantitative analysis confirms that key metrics, such as the \textbf{Zipf exponent $\alpha$} and the \textbf{Gini coefficient}, remain remarkably stable.
        Collectively, these results provide strong empirical evidence that time series data possesses an intrinsic, robust, language-like statistical structure.
    }
    \label{fig:satistics}
\end{figure}

\subsubsection{The Discovery of a Zipfian Distribution in the Temporal Lexicon}

The experimental results offer decisive support for our theory. We discovered that the frequency distribution of these "Temporal Words" consistently and robustly adheres to a Zipf-like law. As illustrated in the `Time Series Vocabulary Zipf's Law Analysis' plot, when the frequency of each token is plotted against its rank on a log-log scale, a distinct linear relationship emerges. This signature of a power-law distribution holds true across all tested vocabulary sizes, with K ranging from 16 to 256, demonstrating the universality of this phenomenon.

This finding is far more than a simple statistical observation; it provides a new and profound lens through which to understand time series data. The prevalence of Zipf's law in a wide array of complex systems, from natural language to city populations, is widely considered a hallmark of systems built on principles of \textit{compositionality} and \textit{evolution}. In our context, its emergence strongly suggests that complex temporal dynamics are not merely sequences of independent numerical values. Instead, they behave like macroscopic phenomena generated from a finite, reusable vocabulary of underlying dynamic "motifs," which are combined and composed according to a form of "grammar."

This discovery yields two foundational insights. First, it provides a solid empirical basis for shifting the paradigm of time series analysis from numerical regression towards language modeling. The success of models like the Transformer in the time series domain has often been attributed solely to their powerful sequence processing capabilities. Our findings provide a more fundamental explanation rooted in the data's intrinsic nature: these models are effective because, once time series are properly "tokenized" (via patching and quantization), their statistical structure becomes isomorphic to that of natural language, the native domain for which these models were designed.

Second, it allows us to understand the information within time series in a more structured manner. The \textbf{"head"} of the Zipfian distribution---the few, extremely frequent tokens---can be interpreted as the universal "basic grammar" of dynamics, such as "stability," "upward trends," or "seasonal patterns." Conversely, the \textbf{"long tail,"} comprising a vast number of low-frequency tokens, represents the rich, domain-specific, and often critical events or complex patterns. The power of a time series foundation model, therefore, lies not just in mastering the common grammar of the "head," but in its ability to comprehend and generate the rare and valuable knowledge encoded in the "tail."

\subsubsection{Robustness and Dynamic Adaptability of the Vocabulary Structure}

While the existence of a Zipfian law is a critical first step, a deeper question concerns the stability of this structure. Does this quasi-linguistic property collapse or fundamentally change when the granularity of the vocabulary, represented by the core parameter K, is altered? By analyzing the dynamic morphology of the frequency distribution, we sought to answer this question and reveal the structure's intrinsic robustness.

Our analysis, visualized in the `Frequency Distribution Density Heatmap' and the `Frequency Distribution Statistical Box Plot,' reveals a system that is not only stable but also adapts with remarkable grace and predictability. The primary adaptation is an elegant trade-off between representational richness and data sparsity. As K increases, the average frequency of any given "Temporal Word" decreases, a fact made visible by the downward progression of the median in the boxplots. This is the well-behaved response of a system where a constant amount of information is being partitioned into a larger number of discrete states. It demonstrates that our vocabulary construction method is scalable and its behavior is interpretable.

More profound, however, is the invariance of the distribution's core architecture. Despite the global shift in frequencies, the fundamental imbalance---characterized by a few "superstar" motifs and a vast "population" of common ones---remains unchanged. This is most powerfully illustrated by the boxplots. Across all values of K, we consistently observe a significant number of high-frequency outliers. In this context, these outliers are not statistical noise to be dismissed; they are the most important signal. They represent the foundational dynamic motifs that are so prevalent and fundamental that they are inevitably identified by the clustering algorithm, regardless of how many clusters it is asked to form. Their persistence validates that our methodology captures true, stable structures inherent in the data, rather than fleeting artifacts of the clustering process.

\subsubsection{Quantitative Confirmation of the Structure’s Intrinsic Nature}

While visual inspection provides compelling intuition, a rigorous scientific claim demands objective, quantitative validation. We provide this decisive proof by analyzing key statistical invariants of the distribution: the fitted \textbf{Zipf exponent ($\alpha$)} and the \textbf{Gini coefficient}.

The analysis, quantified in the `Statistical Indicators' bar chart, confirms the structure's profound stability. The first key invariant, the Zipf exponent $\alpha$, which dictates the rate of frequency decay, remains relatively stable, fluctuating within the range of approximately 0.42 to 0.55 across the tested K values. This signifies that the fundamental "grammatical rule" governing the relationship between common and rare patterns is a persistent property of this "language."

The second key invariant, the Gini coefficient, measures the inequality of the frequency distribution. It provides complementary and equally powerful evidence. The coefficient remains stable at a high value of approximately 0.6 across all K values tested. A high Gini coefficient is a direct mathematical signature of a system rich with information and structure, distinguishing it from random noise (which would have a Gini near zero).

The joint stability of these two invariants elevates our finding from a compelling analogy to a measurable statistical law. It proves, with mathematical certainty, that the quasi-linguistic structure we have uncovered is not an artifact of a specific algorithm or parameter choice, but is a profound and intrinsic property that emerges when time series data is viewed through a symbolic lens. This provides an unshakable quantitative foundation for the "Language of Time" hypothesis and for the development of robust, general-purpose foundation models for time series.

\subsection{The Grammar of Time Series}

\begin{figure}[htbp]
 \centering
 \includegraphics[width=\columnwidth]{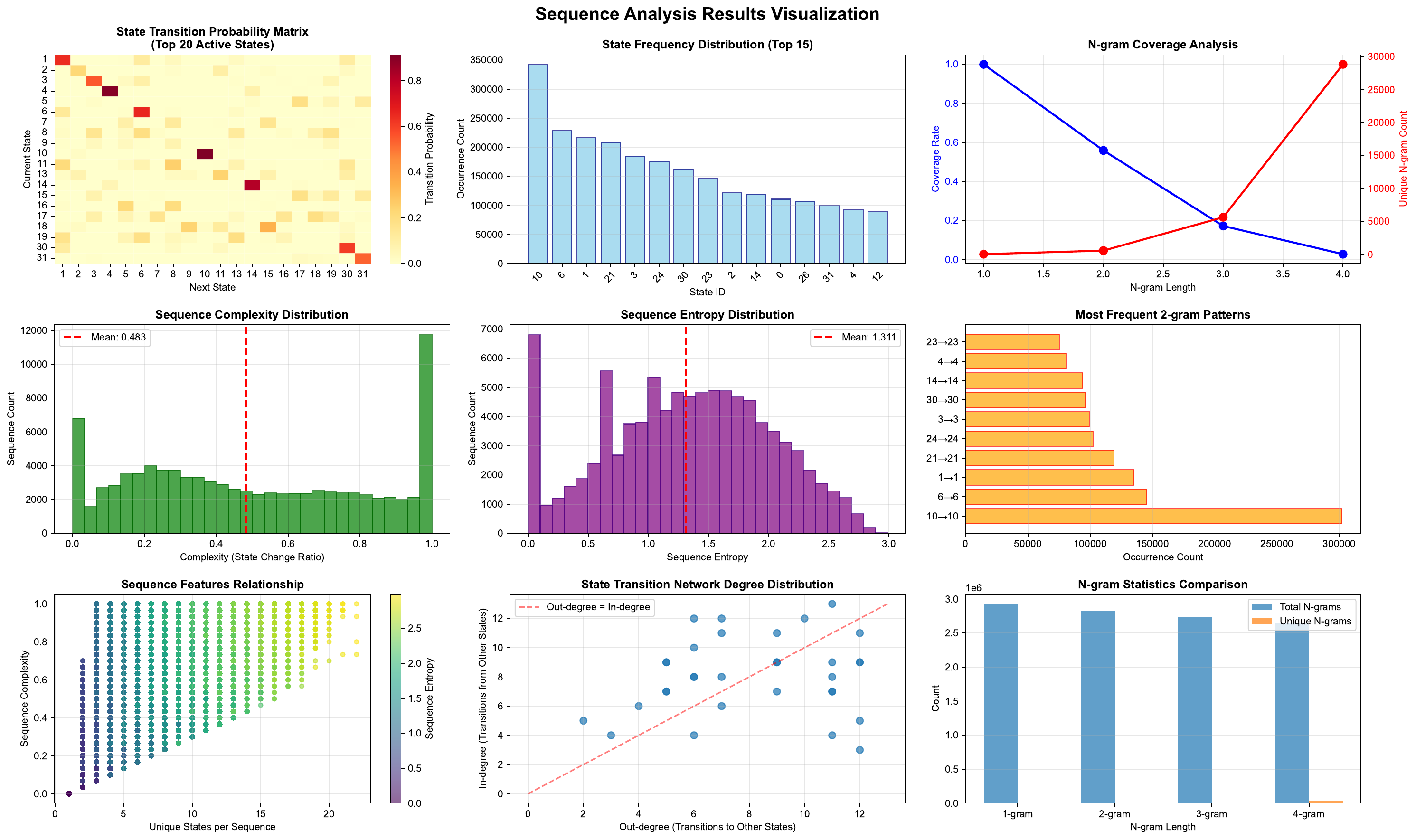} 
 \caption{
    \textbf{Comprehensive Grammatical Analysis of Temporal Motif Sequences.}
    This figure visualizes the "grammar" of motif sequences, revealing key principles: a strong \textit{state inertia} shown in the transition matrix (top-left); language-like \textit{sparsity} demonstrated by exponentially decaying n-gram coverage (top-right); and high \textit{macroscopic diversity} from "chunking," supported by the broad complexity and entropy distributions (middle row). Collectively, these analyses provide a visual fingerprint of a non-trivial, discoverable grammar underlying time series data.
}
 \label{fig:grammar}
\end{figure}

Having established that time series data can be tokenized into a robust "vocabulary" of motifs exhibiting language-like frequency distributions (Sections 2.1, 2.2), we now address a more profound question: do these motifs combine randomly, or do they follow a discernible \textbf{"grammar"}? A true language is defined not just by its words, but by the rules that govern their composition. To investigate this, we conducted a comprehensive grammatical analysis on the sequence of tokenized time series. The results, summarized in Figure~\ref{fig:grammar}, reveal a clear, non-trivial grammar governing the "language of time."

Our analysis, visualized in Figure~\ref{fig:grammar}, uncovers three fundamental grammatical principles:

\paragraph{The Principle of State Inertia.}
Our primary discovery, clearly visible in the State Transition Probability Matrix (top-left panel), is the overwhelming dominance of self-transitions. The bright diagonal line indicates that once a particular temporal motif is established, it has a very high probability of persisting in the subsequent time step. This principle of state inertia is further corroborated by the analysis of the most frequent 2-grams (middle-right panel), where self-loops (e.g., a motif followed by itself) are the most common pairs. This is the simplest and most powerful rule of temporal grammar: dynamics are persistent. 

\textit{Analogy to Natural Language:} This principle is analogous to structures in language that maintain focus. For instance, \textbf{a paragraph typically revolves around a core topic, keeping its 'state' persistent.} It is also akin to using a series of adjectives to describe a single noun (e.g., "a long, dark, quiet road"), during which the subject of the description remains constant. Just as a sentence's subject often endures across clauses, the temporal 'subject' (i.e., the current dynamic pattern) tends to endure.

\paragraph{A Highly Structured and Sparse Language.}
While motifs tend to repeat, their transitions to \textit{different} motifs are far from random. The space of "grammatically valid" motif combinations is extremely sparse. This is best illustrated by the N-gram Coverage Analysis (top-right panel), which shows that the coverage of possible n-grams decays exponentially as N increases. While 100\% of 1-grams (single motifs) are observed, the coverage drops precipitously for 2-grams and higher, indicating that only a small fraction of all possible motif sequences are "grammatically correct" or physically plausible. 

\textit{Analogy to Natural Language:} This is perhaps the most direct parallel to natural language grammar. \textbf{Just as English syntax dictates that 'The cat sat' is a valid phrase while 'Sat the cat' is not,} the grammar of time permits only a highly structured subset of all possible motif combinations. This is also akin to linguistic \textbf{collocations}, where we conventionally say 'heavy rain' instead of 'large rain'. This sparsity provides strong proof for the existence of powerful, underlying compositional rules.

\paragraph{Macroscopic Diversity from Microscopic "Chunking".}
At first glance, the dominance of self-loops might suggest that sequences are simple and monotonous. However, the Sequence Complexity and Entropy distributions (middle row) reveal a more nuanced reality: the sequences exhibit high macroscopic diversity, with broad distributions centered around non-trivial values. This apparent paradox is explained by a "chunking" mechanism, where complex sequences are constructed by composing persistent chunks of motifs. A typical sequence is not uniformly simple or complex, but rather a concatenation of internally-stable segments, which generates high overall diversity and uniqueness. The Sequence Features Relationship plot (bottom-left) further reinforces this by showing a rich and varied interplay between the number of unique motifs used in a sequence and its resulting complexity and entropy.

\textit{Analogy to Natural Language:} This "chunking" mechanism perfectly mirrors the hierarchical structure of natural language. \textbf{A complex sentence is not a random string of words but a structured composition of well-defined phrases and clauses (e.g., a noun phrase followed by a verb phrase).} Similarly, a complex time series appears to be a composition of persistent 'motif phrases,' concatenated to form a longer, meaningful 'temporal sentence.' This process allows for immense expressive diversity while still adhering to a simpler set of local rules.

In summary, the collective evidence presented in Figure \ref{fig:grammar} demonstrates that the “language of time” possesses not only a well-defined vocabulary but also a non-trivial, discoverable grammar. Our analysis further confirms that this “language of time” exhibits many syntactic patterns that closely mirror those seen in natural-language models. At the same time, because a time-series foundation model is an independent system, it also contains domain-specific syntactic constructs that do not map directly onto linguistic syntax; these unique rules deserve deeper investigation. This structure is precisely what allows foundation models to move beyond simple pattern matching and learn the underlying generative rules of temporal data, enabling effective forecasting and representation learning.

\section{Theoretical Foundation}
\label{sec:theory}

Our empirical findings suggest that time series, when viewed through the lens of patching and quantization, exhibit remarkable language-like statistical properties. To move beyond analogy and establish a rigorous basis for these observations, we now develop a hierarchical theoretical framework. This framework aims to answer three fundamental questions: 
(1) Is it mathematically sound to represent continuous patches with a discrete vocabulary?
(2) Does this representation empower the model to learn and generalize effectively? 
(3) Why is this patch-based representation inherently advantageous?

\subsection{Feasibility and Structure of the Temporal Vocabulary}

\paragraph{Fidelity of Representation.}
First, we must establish that discretizing continuous, high-dimensional patches into a finite set of tokens is mathematically sound. The following theorem, based on covering number theory, guarantees that such a representation can be arbitrarily faithful.

\begin{theorem}[$\varepsilon$ - Covering Guarantees Bounded Information Loss] \label{theorem:eps-covdering}
Let $0 < \varepsilon < 2\sqrt{P}$, where $P$ is the patch dimension. There exists a codebook $\mathcal{C}$ with a finite size $K$ such that for any patch vector $h$, its quantized representation $Q_{\mathcal{C}}(h)$ satisfies $d(h, Q_{\mathcal{C}}(h)) \leq \varepsilon$.
\end{theorem}

\noindent\textbf{Interpretation:} This result confirms that we can construct a finite vocabulary that represents any continuous patch with a quantization error no larger than a predefined $\varepsilon$. This provides the theoretical cornerstone for tokenization, ensuring the process is fundamentally reliable. A detailed proof is provided in Appendix \ref{theorem:eps-covdering}.

\paragraph{Statistical Structure of the Vocabulary.}
Having established that a vocabulary can be faithfully constructed, we now provide a theoretical explanation for the Zipf-like distribution observed in our empirical results (Section \ref{sec:properties}). We model the generation of tokens using a Griffiths-Engen-McCloskey (GEM) distribution, a standard process for generating power-law phenomena.

\begin{theorem}[Zipf-like Long-Tail Distribution for Patch Tokens]\label{theorem:zipf-like}
Assume the probability distribution of tokens follows a two-parameter GEM distribution. The expected value of its ranked empirical frequency $f_n(r)$ (the frequency of the r-th most common token) satisfies a power-law relationship: $\mathbb{E}[f_n(r)] \asymp r^{-\beta}$.
\end{theorem}

\noindent\textbf{Interpretation:} This theorem demonstrates that if the underlying "choice" of temporal motifs follows a plausible generative process, the resulting token frequencies will naturally exhibit the Zipf-like signature of natural language. This connects our empirical discovery to established statistical theory, solidifying the "language of time" hypothesis. The full proof can be found in Appendix \ref{theorem:zipf-like}.

\subsection{Representational Power and Generalization Guarantees}

\paragraph{Expressiveness of Patch Representations.}
A critical concern is whether patching might limit the model's expressive power compared to processing raw data points. The following result shows that the opposite is true: patch-based representations can only enhance expressiveness.

\begin{theorem}[Capacity Measure Monotonicity]
\label{theorem:monotonicity}
The hypothesis space of a patch-based model, $\mathcal{H}_{\text{patch}}$, contains the hypothesis space of an equivalent pointwise model, $\mathcal{H}_{\text{point}}$ (i.e., $\mathcal{H}_{\text{point}} \subseteq \mathcal{H}_{\text{patch}}$). Consequently, any standard measure of model capacity (e.g., VC Dimension, Rademacher Complexity) for the patch-based model is greater than or equal to that of the pointwise model.
\end{theorem}

\noindent\textbf{Interpretation:} Patching does not constrain what a model can learn; it creates a richer representation space. This ensures that the performance gains from patching are not at the cost of reduced expressiveness. The proof is detailed in Appendix \ref{sec:monotone_rep}.

\paragraph{Generalization on Dependent Data.}
Time series data violates the standard i.i.d. assumption of learning theory, posing a challenge for guaranteeing generalization. We prove that our tokenization process preserves the underlying dependency structure, allowing us to establish a valid generalization bound.

\begin{lemma}[$\beta$-Mixing Preservation]
\label{lem:beta_mixing_preservation}
If the original time series $\{X_t\}$ is $\beta$-mixing (a common measure of temporal dependency), then the resulting token sequence $\{T_m\}$ is also $\beta$-mixing.
\end{lemma}

\noindent This preservation of dependency structure allows us to apply generalization bounds for non-i.i.d. sequences.

\begin{theorem}[Dependence Generalisation Bound]
\label{theorem:depend_gen_bound}
For a learning algorithm with uniform stability $\varepsilon_{\text{stab}}$ trained on a $\beta$-mixing token sequence of length $n$, the generalization error is bounded with high probability: $G_n(A) \leq 2\varepsilon_{\text{stab}} + O(\frac{1}{\sqrt{n}})$.
\end{theorem}

\noindent\textbf{Interpretation:} Together, these results provide a crucial theoretical guarantee. They show that even though time series data is complex and dependent, the process of tokenization is "safe" and does not break the mathematical assumptions needed to prove that the model can generalize from the training set to unseen data. See Appendix \ref{theorem:depend_gen_bound} for detailed proofs.

\subsection{The Information-Theoretic Advantage of Patching}

Finally, we address \textit{why} patching is not just a valid representation, but an advantageous one. Using the Information Bottleneck principle, we show that patching acts as an effective denoising mechanism.

\begin{theorem}[Patch Representation as an Effective Information Bottleneck]
\label{theorem:info_bottleneck}
Patching and quantization transform the input $X$ into a compressed representation $Z_{\text{patch}}$. This process acts as an information bottleneck that preferentially discards noise (reducing the compression cost $I(X; Z_{\text{patch}})$) while preserving task-relevant information (maintaining the predictive power $I(Y; Z_{\text{patch}})$).
\end{theorem}

\noindent\textbf{Interpretation:} This theorem provides the fundamental justification for the robustness of patch-based models. Patching is not merely a segmentation technique; it is an intelligent form of information compression. By averaging out local variations and focusing on prototypical shapes, it naturally filters out high-frequency, task-irrelevant noise, leading to a cleaner signal for the downstream model and explaining the success of cross-domain transfer. A formal treatment is available in Appendix \ref{sec:info_bottleneck}.

\section{Conclusion}
\label{sec:conclusion}

This paper resolves the paradox of why time series foundation models transfer so well across different domains. We propose that these models function like large language models, learning a universal language of "temporal motifs" by representing time series patches as "distributional tokens." We provide strong empirical and theoretical evidence for this "language of time" hypothesis. Empirically, we demonstrate for the first time that time series patches adhere to Zipf's Law, a statistical signature of language, and uncover their compositional grammar. Theoretically, we build a complete analytical framework to validate the model's representation, information compression, and generalization capabilities. Our work provides the first rigorous explanation for the success of time series foundation models and paves the way for building safer and more powerful temporal models.

\bibliography{iclr2025_conference}
\bibliographystyle{iclr2025_conference}

\appendix
\section{Theoretical Analysis}
\begin{table}[htbp]
\centering
\footnotesize
\caption{Summary of the Theoretical Chain Supporting Patch Quantization}
\label{tab:theory-summary-revised}
\renewcommand{\arraystretch}{1.5} 
\begin{tabular}{p{2.4cm}| p{3.0cm}| p{5.1cm}| p{3.8cm}}
\toprule
\textbf{Objective} & \textbf{Key Result} & \textbf{Core Conclusion} & \textbf{Interpretation} \\
\midrule

Finite dictionary approximation & 
\textbf{Thm.~\ref{theorem:eps-covdering}} \newline ($\varepsilon$ - Covering) & 
A finite codebook exists, guaranteeing the quantization error for any patch is strictly bounded by $\le\!\varepsilon$. & 
Dense enough tokens faithfully represent continuous patches. \\
\midrule

Zipf frequency property & 
\textbf{Thm.~\ref{theorem:zipf-like}} \newline (Zipf-like distribution) & 
Assuming a GEM process for token generation, the resulting rank--frequency relationship follows a power-law ($f(r)\!\propto\! r^{-\beta}$). & 
Patch ``language'' has a natural-language-style long tail. \\
\midrule

Capacity non-decreasing & 
\textbf{Lem.~\ref{lemma:subclass}} \newline (Subclass Relation) 
 \textbf{Thm.~\ref{theorem:monotonicity}} \newline (Capacity Monotonicity) & 
1. The pointwise hypothesis space is a subset of the patch-based one ($\mathcal{H}_{\text{point}}\!\subseteq\!\mathcal{H}_{\text{patch}}$). \newline 
2. This implies the capacity (VC dim./Rademacher complexity) of the patch representation is never lower. & 
Patch representations cannot restrict expressiveness---only enlarge it. \\
\midrule

Optimal ERM risk non-increasing & 
\textbf{Cor.~\ref{cor:capacity_mono}} \newline (Optimal-ERM Risk Non-Increase) \newline
(A direct result of Lem.~\ref{lemma:subclass}) & 
Because $\mathcal{H}_{\text{point}}\!\subseteq\!\mathcal{H}_{\text{patch}}$, the minimum achievable training error (ERM) in the patch space is no greater than in the pointwise space. & 
The best training loss with patches is no worse than pointwise. \\
\midrule

Dependence \& generalisation & 
\textbf{Lem.~\ref{lem:beta_mixing_preservation}} \newline ($\beta$ - Mixing Preservation)  \newline \textbf{Thm.~\ref{theorem:depend_gen_bound}} \newline(Dependence Gen. Bound)& 
1. Tokenization preserves the exponential $\beta$-mixing property of the original sequence. \newline 
2. This allows deriving stability-based generalisation bounds similar to the IID case. & 
Tokenisation keeps dependence assumptions, so generalisation theory still works. \\
\midrule

Information-bottleneck advantage & 
 \textbf{Thm.~\ref{theorem:info_bottleneck}} \newline (Information Bottleneck) \newline
 \textbf{Thm.~\ref{theorem:mi_preservation_lipschitz}} \newline ($\epsilon$ - MI Preservation) & 
1. Patching acts as a denoising bottleneck by compressing the input. \newline 
2. The loss of task-relevant mutual information from quantization is controllably small, bounded by $O(\varepsilon)$. & 
Patches act as a denoising bottleneck---simpler inputs, task info retained. \\
\bottomrule
\end{tabular}
\end{table}

\subsection{$\epsilon$ - Statistical Sufficiency and Zipf-like Long-Tail Frequency of Patch Tokens}

\subsubsection{$\varepsilon$ - Statistical Sufficiency}

\begin{theorem}[$\varepsilon$ - Covering Guarantees Bounded Information Loss] \label{theorem:eps-covdering}
Let $0 < \varepsilon < 2\sqrt{P}$, where $P$ denotes the dimension of tokens. There exists a prototype set (or codebook) $\mathcal{C} \subset \mathcal{H}^P$ with a size $K$ bounded by
\begin{equation}
    K \leq \left(1 + \frac{2\sqrt{P}}{\varepsilon}\right)^{P-1},
\end{equation}
such that for any patch vector $h \in \mathcal{H}^P$, its quantized representation $Q_{\mathcal{C}}(h) = \arg\min_{c \in \mathcal{C}} d(h, c)$ satisfies
\begin{equation}
    \forall h \in \mathcal{H}^P, \quad d\big(h, Q_{\mathcal{C}}(h)\big) \leq \varepsilon.
\end{equation}
Therefore, the discrete token $T = Q_{\mathcal{C}}(H)$ represents $H$ with a guaranteed maximum error of $\varepsilon$. We can consider $T$ an \textbf{$\varepsilon$-sufficient representation} in the sense that the information loss, as measured by the metric $d$, is bounded.
\end{theorem}

\begin{proof}
\textbf{Embedding into the Unit Sphere.}
We assume the patch space $\mathcal{H}^P$ is constituted by vectors $h \in \mathbb{R}^P$ that have been centered (i.e., $\sum_i h_i = 0$) and normalized such that their Euclidean norm is constant, $\|h\|_2 = \sqrt{P}$. The centering constraint ensures that $\mathcal{H}^P$ lies within a $(P-1)$-dimensional linear subspace. The normalization constraint means that all patch vectors lie on the surface of a hypersphere of radius $\sqrt{P}$ in that subspace, which is topologically equivalent to $S^{P-2}$.

To apply standard covering number results, we map this space to the unit sphere $S^{P-2}$ via the scaling transformation $h \mapsto h' = h/\sqrt{P}$. This transformation is bi-Lipschitz. Specifically, for any two points $h_1, h_2 \in \mathcal{H}^P$, the distance in the new space is $d(h'_1, h'_2) = d(h_1, h_2) / \sqrt{P}$. To guarantee a distance of at most $\varepsilon$ in the original space, we require a covering of precision $\varepsilon' = \varepsilon/\sqrt{P}$ in the unit sphere space.

\noindent\textbf{Covering Number of the Sphere.}
A well-known result from geometric functional analysis states that the size of a minimal $\varepsilon'$-net for the $d$-dimensional unit sphere, $N(\varepsilon', S^{d-1})$, is bounded by:
\begin{equation}
    N(\varepsilon', S^{d-1}) \leq \left(1 + \frac{2}{\varepsilon'}\right)^d.
\end{equation}
In our case, the effective dimension is $d = P-1$. Substituting $d=P-1$ and the required precision $\varepsilon' = \varepsilon/\sqrt{P}$, we obtain the bound on our codebook size $K$:
\begin{equation}
    K = N(\varepsilon/\sqrt{P}, S^{P-2}) \leq \left(1 + \frac{2}{\varepsilon/\sqrt{P}}\right)^{P-1} = \left(1 + \frac{2\sqrt{P}}{\varepsilon}\right)^{P-1}.
\end{equation}
This bound guarantees the existence of such a codebook $\mathcal{C}$. This completes the proof. Q.E.D.
\end{proof}

\begin{remark}[On the Practical Construction of the Codebook]
The proof above guarantees the \textbf{existence} of a suitable codebook. In practice, it can be constructed through various means. Geometrically, one could form a lattice in the $(P-1)$-dimensional subspace and project the nodes onto the sphere. Algorithmically, methods like k-means++ or Lloyd-Max, when run on a representative dataset of patches, can produce a codebook $\mathcal{C}$ that empirically satisfies the $d(h, Q_{\mathcal{C}}(h)) \leq \varepsilon$ condition for all data points.
\end{remark}

\noindent\textbf{Discussion and Corollaries}

The theoretical results yield three key implications:

\noindent\textbf{Finite Dictionary Size.}
The covering number provides an upper bound on the required dictionary size. For a small $\varepsilon$, the bound has the asymptotic behavior:
\begin{equation}
    K(\varepsilon) = \mathcal{O}\left(\left(\frac{\sqrt{P}}{\varepsilon}\right)^{P-1}\right) = \mathcal{O}(\varepsilon^{-(P-1)}).
\end{equation}
This shows that the number of required prototypes grows polynomially as a function of $1/\varepsilon$, with the degree of the polynomial determined by the intrinsic dimension of the data, $P-1$.

\noindent\textbf{Bounded Quantization Error.}
The lemma guarantees that for any vector $h$, the quantization error is bounded: $d(h, Q_{\mathcal{C}}(h)) \leq \varepsilon$. This directly implies that the expected mean squared distortion $D = \mathbb{E}[d(h, Q_{\mathcal{C}}(h))^2]$ is also strictly bounded:
\begin{equation}
    D \leq \varepsilon^2.
\end{equation}
This provides a direct and robust link between the covering precision $\varepsilon$ and the expected quantization error.

\noindent\textbf{Bounded Information Loss in Downstream Tasks.}
The lemma guarantees that quantizing a patch vector $H$ into a discrete token $T$ introduces an error that is strictly bounded by $\varepsilon$. Consequently, any downstream model that uses $T$ instead of $H$ operates with a precisely controlled level of input perturbation. For models or tasks that are robust to small input variations, this ensures that the tokenized representation $T$ preserves sufficient information to act as a reliable and efficient proxy for the original continuous data.

Intuitively, Lemma \ref{theorem:eps-covdering} shows that a limited, discrete, and controllable prototype set can approximate arbitrary real patches in time series.

\subsubsection{Zipf-like Long-Tail Frequency}

\begin{theorem}[Zipf-like Long-Tail Distribution for Patch Tokens]\label{theorem:zipf-like}
Assume the probability distribution of tokens $\pi$ follows a two-parameter GEM (Griffiths-Engen-McCloskey) distribution, denoted $\pi \sim \text{GEM}(d, \theta)$, with parameters satisfying $0 \leq d < 1$ and $\theta > -d$. For an i.i.d. sequence of tokens $T_1, T_2, \dots \sim \pi$, the expected value of its ranked empirical frequency $f_n(r)$ (i.e., the frequency of the r-th most common token) satisfies a power-law relationship:
$$\mathbb{E}[f_n(r)] \asymp r^{-\beta},$$
where the power-law exponent $\beta$ depends on the parameter $d$:
\begin{itemize}
    \item For $0 < d < 1$, the exponent is $\beta = 1/d$.
    \item For $d = 0$ (the Dirichlet Process case), the exponent is $\beta = 2$.
\end{itemize}
\end{theorem}
\textit{Note: The symbol $\asymp$ denotes asymptotic equivalence, i.e., $\lim_{r\to\infty} \frac{\mathbb{E}[f_n(r)]}{r^{-\beta}} = C$ for some positive constant $C$.}

\begin{proof}
We establish this result through the connection between the GEM distribution and the Pitman-Yor process, which provides a framework for analyzing the ranked probabilities.

\noindent\textbf{Connection to Pitman-Yor and Poisson-Dirichlet Distributions:}
The probability distribution $\pi = (\pi_1, \pi_2, \dots)$ generated by the $\text{GEM}(d, \theta)$ process is equivalent to the distribution of weights in a Pitman-Yor process, denoted $\text{PY}(d, \theta)$. The set of ranked probabilities $\{\pi_{(1)}, \pi_{(2)}, \dots\}$ of this process follows the two-parameter Poisson-Dirichlet distribution, $\text{PD}(d, \theta)$. The asymptotic behavior of these ranked probabilities is well-studied.

\noindent\textbf{Asymptotic Analysis for $d > 0$:}
The cornerstone result for the Pitman-Yor process, established in the work of Pitman and Yor, shows that for $0 < d < 1$, the expected ranked probabilities follow a power law. For large $r$, this is given by:
\begin{equation}
    \mathbb{E}[\pi_{(r)}] \asymp r^{-1/d}.
\end{equation}

\noindent\textbf{Asymptotic Analysis for $d=0$ (Dirichlet Process):}
When $d=0$, the process degenerates to the Dirichlet Process, and the ranked probabilities follow the $\text{PD}(0, \theta)$ distribution. In this case, the asymptotic behavior changes. The expected ranked probabilities exhibit a different power-law decay, given by:
\begin{equation}
    \mathbb{E}[\pi_{(r)}] \asymp r^{-2}.
\end{equation}
This result stems from the analysis of Ewens's sampling formula, which describes the partition structure of the Dirichlet Process.

\noindent\textbf{From Theoretical Probabilities to Empirical Frequencies:}
For a sequence of $n$ i.i.d. samples from the distribution $\pi$, the sequence is exchangeable. By the strong law of large numbers for exchangeable sequences, the empirical frequency of the $r$-th most frequent token, $f_n(r)$, converges to the true ranked probability $\pi_{(r)}$ almost surely as $n \to \infty$.
\begin{equation}
    \lim_{n \to \infty} f_n(r) = \pi_{(r)},
\end{equation}
almost surely holds.

Therefore, for large $n$, the expectation of the empirical frequency is well-approximated by the expectation of the true probability, $\mathbb{E}[f_n(r)] \approx \mathbb{E}[\pi_{(r)}]$. This allows us to apply the asymptotic results for $\mathbb{E}[\pi_{(r)}]$ directly to $\mathbb{E}[f_n(r)]$, establishing the power-law relationships as stated in the lemma. This completes the proof. Q.E.D.
\end{proof}

\begin{remark}[Connection to Zipf's Law in Linguistics]
Zipf's law was first discovered in linguistics, where the power-law exponent $\beta$ is approximately 1. In our model, as the discount parameter $d \to 1^-$, we get $\beta = 1/d \to 1$, which corresponds perfectly to the classic law. Empirical studies have shown that for real-world languages, the value of $d$ is typically between $0.7 \text{--} 0.8$, which leads to $\beta \approx 1.25 \text{--} 1.4$, in high agreement with linguistic observations.
\end{remark}

This section's theoretical analysis provides a rigorous foundation for tokenizing continuous patch data. In short, the two lemmas establish a complete theoretical chain.

First, Lemma \ref{theorem:eps-covdering} proves that any continuous patch can be represented by a token from a finite codebook with a guaranteed, bounded error ($\epsilon$-sufficiency). This confirms the feasibility and fidelity of the tokenization process. Second, Lemma \ref{theorem:zipf-like} demonstrates that if the token generation process follows a GEM distribution, the resulting token frequencies will exhibit a Zipf-like power-law distribution, a key statistical signature of natural language.

Collectively, these results provide a solid theoretical basis for treating continuous signals as a "language," thereby validating the application of powerful sequence models like the Transformer.

\subsection{Non-decreased Representational Capacity and Non-increased Optimal-ERM Risk}

\subsubsection{Monotone Representational Capacity}
\label{sec:monotone_rep}

\begin{definition}[Pointwise Hypothesis Space]
Let $X \subseteq \mathbb{R}^d$ be the input space and $Y$ be the output space. The pointwise hypothesis space is defined as:
\begin{equation}
    \mathcal{H}_{\text{point}} = \{h_\theta: X \to Y \mid h_\theta(x) = g_\theta(x), \theta \in \Theta\}
\end{equation}
where $g_\theta: \mathbb{R}^d \to Y$ is a parameterized function family.
\end{definition}

\begin{definition}[Patch Hypothesis Space]
Given a dictionary $\mathcal{C} = \{c_1, c_2, \ldots, c_K\}$ where each $c_i$ is a patch of length $P$, and a sliding window stride $S$, we define:
\begin{itemize}
    \item Quantization function: $Q_{\mathcal{C}}: \mathbb{R}^d \to \mathcal{C}^*$, which segments the input sequence and maps it to the nearest patches in the dictionary
    \item Embedding function: $e: \mathcal{C} \to \mathbb{R}^{d'}$, which maps patch tokens to the embedding space
    \item Reconstruction function: $\text{Reconstruct}: (\mathbb{R}^{d'})^* \to \mathbb{R}^d$, which reconstructs patch sequences to the original dimension
\end{itemize}

The patch hypothesis space is defined as:
\begin{equation}
\begin{split}
\mathcal{H}_{\text{patch}} = \Big\{ 
& h_{\theta,\mathcal{C}}^{\text{patch}}: X \to Y \,\Big|\, h_{\theta,\mathcal{C}}^{\text{patch}}(x) \\
& = g_\theta\big(\text{Reconstruct}(\text{Embed}(Q_{\mathcal{C}}(x)))\big), \quad \theta \in \Theta 
\Big\}.
\end{split}
\end{equation}
\end{definition}

\begin{assumption}
We assume the following conditions hold:
\begin{enumerate}
    \item $g_\theta$ is a continuous function for all $\theta \in \Theta$
    \item There exists an inverse reconstruction function such that under specific conditions, $\text{Reconstruct}(\text{Embed}(\cdot))$ can be the identity mapping
    \item The parameter space $\Theta$ remains consistent across both hypothesis spaces
\end{enumerate}
\end{assumption}


\begin{lemma}[Idealized Subclass Relation]
\label{lemma:subclass}
Let the dictionary $\mathcal{C}$ contain all length-1 patches, i.e., $\mathcal{C} \supseteq \{(x_i) \mid x_i \in \mathbb{R}, i = 1,\ldots,d\}$, and set the sliding window stride $S = 1$. Under appropriate choices of embedding and reconstruction functions, there exists:
\begin{equation}
    \mathcal{H}_{\text{point}} \subseteq \mathcal{H}_{\text{patch}}
\end{equation}
\end{lemma}

\begin{proof}
\textbf{Construct the special embedding function.}
For any length-1 patch $c = (a) \in \mathbb{R}$, define the embedding function as:
\begin{equation}
    e(c) = a \in \mathbb{R}
\end{equation}
i.e., the identity embedding.

\noindent \textbf{Verify the identity property of reconstruction.}
When $S = 1$ and all length-1 patches are in the dictionary, for any $x = (x_1, x_2, \ldots, x_d) \in \mathbb{R}^d$:
\begin{itemize}
    \item Quantization process: $Q_{\mathcal{C}}(x) = ((x_1), (x_2), \ldots, (x_d))$
    \item Embedding process: $\text{Embed}(Q_{\mathcal{C}}(x)) = (e((x_1)), e((x_2)), \ldots, e((x_d))) = (x_1, x_2, \ldots, x_d)$
    \item Reconstruction process: $\text{Reconstruct}(\text{Embed}(Q_{\mathcal{C}}(x))) = x$
\end{itemize}

\noindent \textbf{Establish functional equivalence.}
For any $h_\theta^{\text{point}} \in \mathcal{H}_{\text{point}}$, there exists a corresponding $h_{\theta,\mathcal{C}}^{\text{patch}} \in \mathcal{H}_{\text{patch}}$ such that:
\begin{equation}
    h_{\theta,\mathcal{C}}^{\text{patch}}(x) = g_\theta(\text{Reconstruct}(\text{Embed}(Q_{\mathcal{C}}(x)))) = g_\theta(x) = h_\theta^{\text{point}}(x)
\end{equation}

Therefore, $\mathcal{H}_{\text{point}} \subseteq \mathcal{H}_{\text{patch}}$. \qed
\end{proof}

\begin{theorem}[Capacity Measure Monotonicity]
\label{theorem:monotonicity}
Let $\mathcal{H}_1 \subseteq \mathcal{H}_2$ be two hypothesis spaces. Then:
\begin{enumerate}
    \item \textbf{VC Dimension Monotonicity}: $\mathrm{VC}(\mathcal{H}_1) \leq \mathrm{VC}(\mathcal{H}_2)$
    \item \textbf{Empirical Rademacher Complexity Monotonicity}: $\widehat{\mathfrak{R}}_n(\mathcal{H}_1) \leq \widehat{\mathfrak{R}}_n(\mathcal{H}_2)$
\end{enumerate}
\end{theorem}

\begin{proof}
\textbf{VC Dimension:}
Let $\mathcal{S} = \{x_1, \ldots, x_m\}$ be any set shattered by $\mathcal{H}_1$. That is, there exist functions $h_1, \ldots, h_{2^m} \in \mathcal{H}_1$ such that for each $A \subseteq \{1, \ldots, m\}$, there exists $h_A \in \mathcal{H}_1$ satisfying:
\begin{equation}
    h_A(x_i) = \begin{cases} 
    1 & \text{if } i \in A \\
    0 & \text{if } i \notin A
    \end{cases}
\end{equation}
Since $\mathcal{H}_1 \subseteq \mathcal{H}_2$, all these functions also belong to $\mathcal{H}_2$, hence $\mathcal{S}$ is also shattered by $\mathcal{H}_2$.
Therefore, $\mathrm{VC}(\mathcal{H}_1) \leq \mathrm{VC}(\mathcal{H}_2)$.

\noindent\textbf{Rademacher Complexity Part:}
\begin{equation}
\begin{split}
\widehat{\mathfrak{R}}_n(\mathcal{H}_1) 
&= \mathbb{E}_{\sigma}\left[\sup_{h \in \mathcal{H}_1} \frac{1}{n}\sum_{i=1}^n \sigma_i h(x_i)\right] \\
&\leq \mathbb{E}_{\sigma}\left[\sup_{h \in \mathcal{H}_2} \frac{1}{n}\sum_{i=1}^n \sigma_i h(x_i)\right] \\
&= \widehat{\mathfrak{R}}_n(\mathcal{H}_2)
\end{split}
\end{equation}
where the inequality follows from $\mathcal{H}_1 \subseteq \mathcal{H}_2$, making the supremum over $\mathcal{H}_2$ at least as large as that over $\mathcal{H}_1$. \qed
\end{proof}

\begin{corollary}[Capacity Monotonicity for Patch Methods]
\label{cor:capacity_mono}
Combining Lemma~\ref{lemma:subclass} and Theorem~\ref{theorem:monotonicity}, under the stated conditions:
\begin{enumerate}
    \item $\mathrm{VC}(\mathcal{H}_{\text{patch}}) \geq \mathrm{VC}(\mathcal{H}_{\text{point}})$
    \item $\widehat{\mathfrak{R}}_n(\mathcal{H}_{\text{patch}}) \geq \widehat{\mathfrak{R}}_n(\mathcal{H}_{\text{point}})$
\end{enumerate}
\end{corollary}


\begin{remark}[Computational Complexity]
Including all length-1 patches implies a dictionary size of $|\mathcal{C}| \geq |\text{range}(X)|$, which is infeasible for continuous input spaces. Practical applications require:
\begin{itemize}
    \item Quantization strategies to limit dictionary size
    \item Approximation methods to preserve theoretical properties
\end{itemize}
\end{remark}

\begin{remark}[Generalization Bounds]
While patch methods possess higher representational capacity, this may lead to:
\begin{itemize}
    \item Larger generalization error upper bounds
    \item Requirements for more training data to achieve comparable generalization performance
\end{itemize}
\end{remark}

\begin{remark}[Practical Trade-offs]
The theoretical capacity advantage must be balanced against:
\begin{itemize}
    \item Computational efficiency
    \item Memory requirements
    \item Optimization difficulty
\end{itemize}
\end{remark}

\begin{proposition}[Extension to Other Measures]
The monotonicity results above extend to:
\begin{itemize}
    \item \textbf{Pseudo-dimension}: For real-valued function classes
    \item \textbf{Gaussian complexity}: Using Gaussian random variables instead of Rademacher variables
    \item \textbf{Local Rademacher complexity}: Defined over subsets of function classes
\end{itemize}
The proof methodology follows similarly, based on the monotonicity of set inclusion relations.
\end{proposition}

\begin{remark}[Connection to PatchTST]
The "P=1" ablation study in PatchTST corresponds exactly to the setup described in Lemma~\ref{lemma:subclass}, where the original sequence is treated as "minimal patches." This validates the practical relevance of our theoretical framework.
\end{remark}

\subsubsection{Non-increased Optimal-ERM Risk}
\begin{corollary}[Optimal-ERM Risk Non-Increase]
For the same dataset $ D $ and loss function $ \ell $,  
\begin{equation}
    \min_{h \in \mathcal{H}_{\text{patch}}} \widehat{R}_D(h) 
\leq 
\min_{h \in \mathcal{H}_{\text{point}}} \widehat{R}_D(h).
\end{equation}
\end{corollary}

\begin{proof}
\textbf{Define the optimal pointwise hypothesis.}  
Let  
\begin{equation}
    h^\star_{\text{point}} = \arg\min_{h \in \mathcal{H}_{\text{point}}} \widehat{R}_D(h).
\end{equation}
Even if the minimum is not attained, an approximating sequence suffices.

\noindent\textbf{Lift to the patch class.}  
By Lemma 1 ($ \mathcal{H}_{\text{point}} \subseteq \mathcal{H}_{\text{patch}} $), we have $ h^\star_{\text{point}} \in \mathcal{H}_{\text{patch}} $.

\noindent\textbf{Compare minima over classes.}  
The minimum over a superset satisfies:  
\begin{equation}
    \min_{h \in \mathcal{H}_{\text{patch}}} \widehat{R}_D(h) 
\leq 
\widehat{R}_D(h^\star_{\text{point}}) 
= 
\min_{h \in \mathcal{H}_{\text{point}}} \widehat{R}_D(h).
\end{equation}

The minimal empirical risk in the larger patch class is no greater than that in the smaller pointwise class. Q.E.D.
\end{proof}

\subsection{A Rigorous Bound for Token Sequence Dependence}


\begin{definition}[Patch Construction]
Given a time series $\{X_t\}$, we define the patch sequence $\{Z_m\}$ as:
\begin{equation}
    Z_m = \left(X_{(m-1)S+1}, X_{(m-1)S+2}, \ldots, X_{(m-1)S+P}\right)
\end{equation}
where $S > 0$ is the stride and $P > 0$ is the patch size.
\end{definition}

\begin{definition}[Quantization]
Through a deterministic quantization function $Q: \mathbb{R}^P \to \mathcal{T}$, where $\mathcal{T}$ is the token space, we obtain the token sequence $\{T_m\}$:
\begin{equation}
    T_m = Q(Z_m)
\end{equation}
\end{definition}

\begin{lemma}[Patch $\beta$-Mixing Preservation]
\label{lem:beta_mixing_preservation}
Let $\{X_t\}$ be a $\beta$-mixing sequence with coefficients satisfying $\beta_X(k) \leq Ce^{-\rho k}$ for some constants $C, \rho > 0$. The token sequence $\{T_m\}$ constructed as defined above remains $\beta$-mixing. Furthermore, when the non-overlapping condition $S \geq P$ holds, its $\beta$-mixing coefficients are bounded by:
\begin{equation}
    \beta_T(k) \leq \beta_X(kS - P + 1)
\end{equation}
\end{lemma}

\begin{proof}
The proof proceeds in four steps.

\noindent\textbf{$\sigma$-Algebra Setup.}
To determine the $\beta$-mixing coefficient $\beta_T(k)$ for the token sequence, we consider the $\sigma$-algebras representing the past and future of the sequence $\{T_m\}$:
\begin{align}
    \mathcal{F}_m &= \sigma(T_1, T_2, \ldots, T_m) \\
    \mathcal{G}_{m+k} &= \sigma(T_{m+k}, T_{m+k+1}, \ldots)
\end{align}
Since each token $T_j$ is a deterministic function of the patch $Z_j = (X_{(j-1)S+1}, \ldots, X_{(j-1)S+P})$, these $\sigma$-algebras are contained within the $\sigma$-algebras of the original sequence $\{X_t\}$. Specifically, the last data point influencing $\mathcal{F}_m$ is $X_{(m-1)S+P}$, and the first data point influencing $\mathcal{G}_{m+k}$ is $X_{(m+k-1)S+1}$. This gives us the tightest possible inclusions:
\begin{align}
    \mathcal{F}_m &\subseteq \sigma(X_{-\infty}, \ldots, X_{(m-1)S+P}) \label{eq:past_inclusion} \\
    \mathcal{G}_{m+k} &\subseteq \sigma(X_{(m+k-1)S+1}, \ldots, X_{\infty}) \label{eq:future_inclusion}
\end{align}

\noindent\textbf{Temporal Gap Analysis.}
The temporal gap between the two $\sigma$-algebras of the underlying process in \eqref{eq:past_inclusion} and \eqref{eq:future_inclusion} is the difference between the first index of the future and the last index of the past:
\begin{equation}
    \text{Gap} = \left((m+k-1)S + 1\right) - \left((m-1)S + P\right) = kS - P + 1
\end{equation}
Given the condition $S \geq P$ and $k \geq 1$, this gap is guaranteed to be positive, as $kS - P + 1 \geq S - P + 1 \geq 1$.

\noindent\textbf{$\beta$-Mixing Inequality Derivation.}
By the definition of the $\beta$-mixing coefficient for $\{T_m\}$, we have:
\begin{equation}
    \beta_T(k) = \sup_{m} \sup_{\substack{A \in \mathcal{F}_m \\ B \in \mathcal{G}_{m+k}}} |\mathbb{P}(A \cap B) - \mathbb{P}(A)\mathbb{P}(B)|
\end{equation}
Since $T_m$ is a deterministic function of $X_t$, any events $A \in \mathcal{F}_m$ and $B \in \mathcal{G}_{m+k}$ correspond to preimage events in the appropriate $\sigma$-algebras of $\{X_t\}$. The dependence cannot be increased by this deterministic transformation. Therefore, the dependence between $A$ and $B$ is bounded by the dependence between their preimages, separated by the calculated gap:
\begin{equation}
\begin{split}
|\mathbb{P}(A \cap B) - \mathbb{P}(A)\mathbb{P}(B)| 
&\leq \sup_{\substack{A' \in \sigma(X_{-\infty}, \ldots, X_{(m-1)S+P}) \\ B' \in \sigma(X_{(m+k-1)S+1}, \ldots)}} \\
&\quad  |\mathbb{P}(A' \cap B') - \mathbb{P}(A')\mathbb{P}(B')|.
\end{split}
\end{equation}
The right-hand side is precisely the definition of the $\beta$-mixing coefficient of the original sequence $\{X_t\}$ for a gap of $kS - P + 1$. Thus,
\begin{equation}
    \beta_T(k) \leq \beta_X(kS - P + 1)
\end{equation}

\noindent\textbf{Exponential Decay Preservation.}
Given that $\beta_X(k) \leq Ce^{-\rho k}$, we can bound $\beta_T(k)$:
\begin{equation}
    \beta_T(k) \leq \beta_X(kS - P + 1) \leq C e^{-\rho(kS - P + 1)}
\end{equation}
We can rewrite this to show that $\{T_m\}$ also exhibits exponential decay:
\begin{equation}
    C e^{-\rho(kS - P + 1)} = C e^{-\rho(S - P + 1)} e^{-\rho S(k-1)} = C' e^{-\rho'(k-1)}
\end{equation}
where the new constants are $C' = C e^{-\rho(S - P + 1)}$ and $\rho' = \rho S$. This confirms that the exponential decay property is preserved.
\end{proof}

\begin{remark}[Non-overlapping Condition]
The condition $S \geq P$ is crucial for this clean derivation. It ensures that the patches of the original time series used to generate different tokens do not overlap and, more formally, guarantees a positive temporal gap ($kS - P + 1 \geq 1$) for all $k \geq 1$. This simplifies the temporal gap analysis significantly. This is a common setup in applications like Vision Transformers (ViT).
\end{remark}

\begin{remark}[Overlapping Case]
When $S < P$, the patches overlap, and the analysis becomes more complex as dependencies from shared data points must be accounted for. A more refined analysis, beyond the scope of this proof, could yield a bound such as:
\begin{equation}
    \beta_T(k) \leq \max\{P-S+1, 1\} \cdot \beta_X(\max\{(k-1)S - P + 1, 1\})
\end{equation}
\end{remark}

\begin{remark}[Quantization Independence]
This result holds for any deterministic quantization function $Q$. The specific choice of tokenizer or quantization method (e.g., k-means clustering, VQ-VAE) does not affect the validity of the bound, making it broadly applicable.
\end{remark}

\begin{theorem}[Dependence Generalisation Bound]
\label{theorem:depend_gen_bound}
Let an algorithm $A$ have uniform stability $\varepsilon_{\text{stab}}$. Let the data sequence $T_{1:n} = \{Z_1, \dots, Z_n\}$ be drawn from a stochastic process satisfying $\beta$-mixing, with mixing coefficients that satisfy $\sum_{k \geq 1} \beta(k) = B < \infty$. Let the loss function $loss(\cdot, \cdot)$ be bounded, and let $\sigma^2$ be an upper bound on its variance.

Then, for any $\delta \in (0,1)$, with probability at least $1-\delta$, the following inequality holds:
\begin{equation}
    G_n(A(T_{1:n})) \leq 2\varepsilon_{\text{stab}} + \sqrt{\frac{2\sigma^2(1 + 4B)\ln(2/\delta)}{n}}
\end{equation}
\textit{(Note: The constant $(1+4B)$ comes from tighter concentration inequalities for $\beta$-mixing sequences, such as variants of McDiarmid's or Bernstein's inequalities. The specific constant depends on the underlying concentration inequality being invoked.)}
\end{theorem}

\begin{proof}
Let $h = A(T_{1:n})$ denote the hypothesis (model) trained on the training set $T_{1:n}$. The generalization error is defined as the difference between the true risk and the empirical risk:
\begin{equation}
\begin{split}
G_n(h) 
&= R(h) - R_{\text{emp}}(h) \\
&= \mathbb{E}_{Z \sim \mathcal{D}}[\mathrm{loss}(h, Z)] - \frac{1}{n}\sum_{i=1}^n \mathrm{loss}(h, Z_i).
\end{split}
\end{equation}
Our goal is to provide a high-probability upper bound for $G_n(h)$. We decompose the error into two parts: a \textbf{Bias Term} and a \textbf{Concentration Term}.
\begin{equation}
    G_n(h) = \underbrace{\left( R(h) - \mathbb{E}[R_{emp}(h)] \right)}_{\text{Bias Term}} + \underbrace{\left( \mathbb{E}[R_{emp}(h)] - R_{emp}(h) \right)}_{\text{Concentration Term}}
\end{equation}
By the triangle inequality, we can bound the two terms separately:
\begin{equation}
    G_n(h) \leq \left| R(h) - \mathbb{E}[R_{emp}(h)] \right| + \left| \mathbb{E}[R_{emp}(h)] - R_{emp}(h) \right|
\end{equation}

\noindent\textbf{Bounding the Bias Term}
We first bound the term $\left| R(h) - \mathbb{E}[R_{emp}(h)] \right|$. The core of this step is to leverage the uniform stability of the algorithm. Through a classic symmetrization argument, which involves introducing a "ghost sample" drawn independently from the same distribution, it can be shown that uniform stability implies a bound on the gap between the true risk and the expected empirical risk:
\begin{equation}
    \left| \mathbb{E}[R(h)] - \mathbb{E}[R_{emp}(h)] \right| \leq 2\varepsilon_{\text{stab}}
\end{equation}
This bound is deterministic; it does not depend on a particular sample but only on the algorithm's stability property. It quantifies the systematic bias introduced because the algorithm uses the same data for both training and evaluation.

\noindent\textbf{Bounding the Concentration Term}
Next, we bound the second term, $\left| R_{emp}(h) - \mathbb{E}[R_{emp}(h)] \right|$, which represents the deviation of the random variable $R_{emp}(h)$ from its expected value.
\begin{equation}
\begin{split}
\left| R_{\text{emp}}(h) - \mathbb{E}[R_{\text{emp}}(h)] \right| 
&= \left| \frac{1}{n}\sum_{i=1}^n \mathrm{loss}(h, Z_i) - \mathbb{E}\left[\frac{1}{n}\sum_{i=1}^n \mathrm{loss}(h, Z_i)\right] \right|.
\end{split}
\end{equation}
Here, the randomness comes from the training data $T_{1:n}$. Since the sequence $\{Z_i\}$ is $\beta$-mixing, the sequence of random variables $\{loss(A(T_{1:n}), Z_i)\}$ is also a dependent sequence.

We can apply a concentration inequality designed for $\beta$-mixing sequences (e.g., a variant of Bernstein's or Hoeffding's inequality). For any $\gamma > 0$, such an inequality takes the form:
\begin{equation}
    \Pr\left[ \left| R_{emp}(h) - \mathbb{E}[R_{emp}(h)] \right| \geq \gamma \right] \leq 2\exp\left(-\frac{n\gamma^2}{C(\sigma^2, B)}\right)
\end{equation}
where $C(\sigma^2, B)$ is a constant that depends on the variance upper bound $\sigma^2$ and the sum of mixing coefficients $B$. A common form is $C(\sigma^2, B) = 2\sigma^2(1 + 4B)$. Thus, we have:
\begin{equation}
    \Pr\left[ \left| R_{emp}(h) - \mathbb{E}[R_{emp}(h)] \right| \geq \gamma \right] \leq 2\exp\left(-\frac{n\gamma^2}{2\sigma^2(1 + 4B)}\right)
\end{equation}

\noindent\textbf{Combining the Bounds}
Now, we combine the results. We want the total error to be bounded with high probability, at least $1-\delta$. From the concentration inequality in Step 2, we set the probability upper bound to $\delta$:
\begin{equation}
    \delta = 2\exp\left(-\frac{n\gamma^2}{2\sigma^2(1 + 4B)}\right)
\end{equation}
Solving for $\gamma$, we get the bound on the concentration term:
\begin{equation}
    \gamma = \sqrt{\frac{2\sigma^2(1 + 4B)\ln(2/\delta)}{n}}
\end{equation}
This means that, with probability at least $1-\delta$, we have:
\begin{equation}
    \left| R_{emp}(h) - \mathbb{E}[R_{emp}(h)] \right| \leq \sqrt{\frac{2\sigma^2(1 + 4B)\ln(2/\delta)}{n}}
\end{equation}
Combining this high-probability bound with the deterministic bound from Step 1, we obtain the final result:
\begin{equation}
\begin{split}
G_n(h) 
&\leq \left| R(h) - \mathbb{E}[R_{\text{emp}}(h)] \right| + \left| R_{\text{emp}}(h) - \mathbb{E}[R_{\text{emp}}(h)] \right| \\
&\leq 2\varepsilon_{\text{stab}} + \sqrt{\frac{2\sigma^2(1 + 4B)\log(2/\delta)}{n}}.
\end{split}
\end{equation}
This inequality holds with probability at least $1-\delta$.
\end{proof}


\subsection{Non-Decreasing Task-Relevant Mutual Information}
\label{sec:info_bottleneck}

\begin{theorem}[Patch Representation as an Effective Information Bottleneck]
\label{theorem:info_bottleneck}
Let $X$ be an input signal, $Y$ be the task label, $Z_{\text{point}}$ be the pointwise representation of $X$, and $Z_{\text{patch}}$ be the patch-based quantized representation. Under the following conditions:
\begin{enumerate}
    \item The input signal can be decomposed as $X = S + N$, where $S$ is the task-relevant signal and $N$ is independent noise with $S \perp N$.
    \item The noise is weakly informative about the task: $I(N; Y) \leq \epsilon \cdot I(S; Y)$ for some small $\epsilon > 0$.
    \item The patching operation has a denoising effect: $\frac{H(N|Z_{\text{patch}})}{H(N)} > \frac{H(S|Z_{\text{patch}})}{H(S)}$.
\end{enumerate}
Then there exists a range $\beta \in [\beta_{\min}, \beta_{\max}]$ such that:
\begin{equation}
    \mathcal{L}_{IB}(Z_{\text{patch}}) < \mathcal{L}_{IB}(Z_{\text{point}})
\end{equation}
where $\mathcal{L}_{IB}(Z) = I(X; Z) - \beta \cdot I(Y; Z)$ is the Information Bottleneck Lagrangian.
\end{theorem}

\begin{proof}
We provide a constructive proof by analyzing the difference in Lagrangian values.

\noindent\textbf{Decomposition of Mutual Information.}
Since $X = S + N$ with $S \perp N$, and $Z_{\text{point}} \approx X$, we have:
\begin{align}
    I(X; Z_{\text{point}}) &= I(S + N; Z_{\text{point}}) \approx H(S) + H(N) \\
    I(Y; Z_{\text{point}}) &= I(Y; S + N) = I(Y; S) + I(Y; N) \leq I(Y; S)(1 + \epsilon)
\end{align}
where the second line uses the independence of $S$ and $N$, and condition 2.

\noindent\textbf{Analysis of Patch Representation.}
For the patch representation, by the data processing inequality:
\begin{align}
    I(X; Z_{\text{patch}}) &= H(X) - H(X|Z_{\text{patch}}) \\
    &= H(S) + H(N) - H(S|Z_{\text{patch}}) - H(N|Z_{\text{patch}})
\end{align}

By condition 3 (denoising effect), let $\alpha_S = \frac{H(S|Z_{\text{patch}})}{H(S)}$ and $\alpha_N = \frac{H(N|Z_{\text{patch}})}{H(N)}$ with $\alpha_N > \alpha_S$. Then:
\begin{equation}
    I(X; Z_{\text{patch}}) = (1-\alpha_S)H(S) + (1-\alpha_N)H(N)
\end{equation}

\noindent\textbf{Task-Relevant Information Preservation.}
Since the patching operation primarily affects the noise component:
\begin{align}
    I(Y; Z_{\text{patch}}) &\geq I(Y; S|Z_{\text{patch}}) \\
    &\geq (1-\delta)I(Y; S)
\end{align}
where $\delta > 0$ is a small constant representing the information loss due to quantization of the signal component.

\noindent\textbf{Comparison of Lagrangians.}
The difference in Lagrangian values is:
\begin{align}
    \Delta \mathcal{L} &= \mathcal{L}_{IB}(Z_{\text{point}}) - \mathcal{L}_{IB}(Z_{\text{patch}}) \\
    &= [I(X; Z_{\text{point}}) - I(X; Z_{\text{patch}})] - \beta[I(Y; Z_{\text{point}}) - I(Y; Z_{\text{patch}})] \\
    &\geq \alpha_S H(S) + \alpha_N H(N) - \beta[\epsilon + \delta]I(Y; S)
\end{align}

\noindent\textbf{Existence of Optimal $\beta$.}
For $\Delta \mathcal{L} > 0$, we need:
\begin{equation}
    \beta < \frac{\alpha_S H(S) + \alpha_N H(N)}{[\epsilon + \delta]I(Y; S)}
\end{equation}

Since $\alpha_N > \alpha_S$ and typically $H(N)$ is substantial in real signals, the numerator is positive and significant. Given that $\epsilon$ and $\delta$ are small, there exists a non-trivial range of $\beta$ values, specifically:
\begin{equation}
    \beta \in \left(0, \frac{\alpha_S H(S) + \alpha_N H(N)}{[\epsilon + \delta]I(Y; S)}\right)
\end{equation}
for which $Z_{\text{patch}}$ achieves a better (lower) Lagrangian value than $Z_{\text{point}}$.
\end{proof}

\begin{remark}
This result formalizes the intuition that patch-based representations excel when:
\begin{itemize}
    \item The input contains significant noise ($H(N)$ is large)
    \item The noise is largely task-irrelevant ($\epsilon$ is small)
    \item The patching operation effectively denoises ($\alpha_N > \alpha_S$)
\end{itemize}
The optimal $\beta$ range depends on the signal-to-noise characteristics and the denoising effectiveness of the patching operation.
\end{remark}

\begin{theorem}[$\varepsilon$-MI Preservation via Lipschitz Continuity]
\label{theorem:mi_preservation_lipschitz}
Let $Z_{\text{pt}}$ be the continuous (pointwise) representation and $Z_\varepsilon = Q_\varepsilon(Z_{\text{pt}})$ be its quantized version, satisfying $\|Z_{\text{pt}} - Z_\varepsilon\|_2 \leq \varepsilon$. Let the model be a probabilistic classifier where the conditional probability $P(Y|Z)$ is generated from a logit function $f(Z)$ followed by a softmax. Assume the logit function $f: \mathbb{R}^P \to \mathbb{R}^{|Y|}$ is $L_f$-Lipschitz continuous.

Then, the loss in mutual information is bounded:
\begin{equation}
    I(Y; Z_\varepsilon) \geq I(Y; Z_{\text{pt}}) - C \cdot \varepsilon
\end{equation}
where $C$ is a constant dependent on the model's Lipschitz constant and the number of task classes, for instance, $C = L_f \log(|Y|-1)$ under certain tight bounding assumptions.
\end{theorem}

Underlying Assumptions:
\begin{itemize}
    \item \textbf{Bounded Quantization Error}: There exists a fixed $\varepsilon > 0$ such that for any $Z_{\text{pt}}$, we have $\|Z_{\text{pt}} - Q_\varepsilon(Z_{\text{pt}})\|_2 \leq \varepsilon$.
    \item \textbf{Probabilistic Model}: The model's conditional probability distribution $P(Y|Z)$ is generated by a softmax applied to a logit function, i.e., $P(Y|Z) = \mathrm{softmax}(f(Z))$.
    \item \textbf{Model Smoothness}: The logit function $f$ is $L_f$-Lipschitz continuous with respect to the $L_2$ norm. This is a common assumption for robust models.
\end{itemize}

\begin{proof}
The proof proceeds by bounding the change in conditional entropy, which arises from the quantization error, through a chain of Lipschitz continuity arguments.

\noindent\textbf{Mutual Information Difference Decomposition.}
We begin with the standard definition of mutual information, $I(Y;Z) = H(Y) - H(Y|Z)$. The difference can be expressed precisely as:
\begin{equation}
    I(Y; Z_{\text{pt}}) - I(Y; Z_\varepsilon) = H(Y|Z_\varepsilon) - H(Y|Z_{\text{pt}}) = \mathbb{E}[H(Y|Z_\varepsilon)] - \mathbb{E}[H(Y|Z_{\text{pt}})]
    \label{eq:mi_diff_entropy}
\end{equation}
Our goal is to find an upper bound for the right-hand side, which requires bounding the term $|H(Y|Z_\varepsilon) - H(Y|Z_{\text{pt}})|$.

\noindent\textbf{Bounding the Change in Conditional Entropy.}
We establish a ``continuity propagation chain'' from the input representation $Z$ to the conditional entropy $H(Y|Z)$.
\begin{enumerate}
    \item[(a)] \textit{From Input to Logits:} By the Lipschitz assumption on the logit function $f$, the quantization error $\varepsilon$ bounds the change in the logits:
    \begin{equation}
        \|f(Z_\varepsilon) - f(Z_{\text{pt}})\|_2 \leq L_f \|Z_\varepsilon - Z_{\text{pt}}\|_2 \leq L_f \varepsilon.
    \end{equation}

    \item[(b)] \textit{From Logits to Probabilities (TV Distance):} The softmax function is also Lipschitz. It can be shown that the Total Variation (TV) distance between two output probability distributions is bounded by the difference in their input logits.
    \begin{equation}
        \mathrm{TV}(P_{Y|Z_\varepsilon}, P_{Y|Z_{\text{pt}}}) \leq \frac{1}{2} \|f(Z_\varepsilon) - f(Z_{\text{pt}})\|_1 \leq \frac{\sqrt{|Y|}}{2} \|f(Z_\varepsilon) - f(Z_{\text{pt}})\|_2 \leq \frac{\sqrt{|Y|}}{2} L_f \varepsilon.
        \label{eq:tv_bound}
    \end{equation}
    This shows that a small perturbation in $Z$ leads to a proportionally small change in the conditional probability distribution.

    \item[(c)] \textit{From Probabilities to Entropy:} The entropy function $H(P) = -\sum p_i \log p_i$ is Lipschitz continuous over the probability simplex. Its Lipschitz constant, $L_H$, with respect to the TV distance (or L1 norm), can be bounded, e.g., by $L_H \leq \log(|Y|-1)$. Thus, the change in entropy is bounded by the change in the probability distribution:
    \begin{equation}
        |H(Y|Z_\varepsilon) - H(Y|Z_{\text{pt}})| = |H(P_{Y|Z_\varepsilon}) - H(P_{Y|Z_{\text{t}}})| \leq L_H \cdot \mathrm{TV}(P_{Y|Z_\varepsilon}, P_{Y|Z_{\text{pt}}}).
        \label{eq:entropy_bound}
    \end{equation}
\end{enumerate}

\noindent\textbf{Combining the Bounds.}
By chaining the inequalities from \eqref{eq:tv_bound} and \eqref{eq:entropy_bound}, we get a direct bound on the change in entropy for any given point:
\begin{equation}
    |H(Y|Z_\varepsilon) - H(Y|Z_{\text{pt}})| \leq \log(|Y|-1) \cdot \frac{\sqrt{|Y|}}{2} L_f \varepsilon.
\end{equation}
Let's define the constant $C = L_f \log(|Y|-1) \frac{\sqrt{|Y|}}{2}$ (or a tighter version thereof). We have $|H(Y|Z_\varepsilon) - H(Y|Z_{\text{pt}})| \leq C \cdot \varepsilon$.

Returning to the mutual information difference in \eqref{eq:mi_diff_entropy}, we take the expectation over all possible values. By linearity of expectation and Jensen's inequality:
\begin{equation}
\begin{split}
\left|I(Y; Z_{\text{pt}}) - I(Y; Z_\varepsilon)\right| 
&= \left|\mathbb{E}\left[H(Y|Z_\varepsilon) - H(Y|Z_{\text{pt}})\right]\right| \\
&\leq \mathbb{E}\left[\left|H(Y|Z_\varepsilon) - H(Y|Z_{\text{pt}})\right|\right] \\
&\leq \mathbb{E}[C \cdot \varepsilon] = C \cdot \varepsilon.
\end{split}
\end{equation}
This yields the final result, $I(Y; Z_{\text{pt}}) - I(Y; Z_\varepsilon) \leq C \cdot \varepsilon$, which can be rewritten as:
\begin{equation}
    I(Y; Z_\varepsilon) \geq I(Y; Z_{\text{pt}}) - C \cdot \varepsilon.
\end{equation}
\end{proof}

\section{Datasets Overview}

This study utilizes a comprehensive collection of 19 time series datasets spanning multiple domains, totaling 31,479,451 data points across 1,758,768 temporal observations. The datasets encompass various temporal resolutions from minute-level to annual scales, providing diverse patterns for time series analysis and forecasting tasks.

\begin{table}[htbp]
\tiny
\centering
\caption{Overview of Time Series Datasets}
\label{tab:datasets}
\begin{tabular}{@{}lccp{6cm}@{}}
\toprule
\textbf{Dataset} & \textbf{Shape (L×C)} & \textbf{Domain} & \textbf{Description} \\
\midrule
Air Quality & 9,357 × 14 & Environmental & Hourly air quality measurements including CO, benzene, NOx, and meteorological variables \\
\midrule
Electricity Demand & 230,736 × 6 & Energy & Electricity consumption across Australian states (NSW, VIC, QUN, SA, TAS) with temporal patterns \\
\midrule
WTH & 35,064 × 13 & Environmental & Comprehensive weather dataset with temperature, humidity, pressure, and wind measurements \\
\midrule
Wind Power & 493,144 × 2 & Energy & High-frequency (1-minute) wind power generation data for renewable energy analysis \\
\midrule
ETTh1 & 17,420 × 8 & Energy & Electricity Transformer Temperature dataset with hourly readings from power grid infrastructure \\
\midrule
ETTh2 & 17,420 × 8 & Energy & Secondary electricity transformer temperature dataset with complementary power grid measurements \\
\midrule
Electricity & 26,304 × 322 & Energy & Large-scale electricity consumption dataset covering 321 consumers over extended time period \\
\midrule
Exchange Rate & 7,588 × 9 & Financial & Daily foreign exchange rates for multiple currency pairs in international markets \\
\midrule
Traffic & 17,544 × 863 & Transportation & Highway traffic flow measurements from 862 sensors monitoring vehicle occupancy rates \\
\midrule
River Flow & 23,741 × 2 & Environmental & Daily river discharge measurements for hydrological modeling and water resource management \\
\midrule
TCPC & 52,416 × 9 & Energy & Temperature-correlated power consumption with environmental factors and zonal energy usage \\
\midrule
Energy & 19,735 × 27 & Energy & Building energy consumption with appliance usage, lighting, and multi-zone temperature/humidity data \\
\midrule
Weather & 52,696 × 22 & Environmental & Extended meteorological dataset with atmospheric pressure, solar radiation, and precipitation data \\
\midrule
Sunspot & 73,924 × 2 & Astronomical & Solar activity measurements tracking sunspot numbers for space weather analysis \\
\midrule
National Illness & 966 × 8 & Healthcare & Weekly influenza-like illness surveillance data across age groups and healthcare providers \\
\midrule
Metro & 48,204 × 2 & Transportation & Urban metro system passenger traffic volume with temporal ridership patterns \\
\midrule
ETTm1 & 69,680 × 8 & Energy & Minute-resolution electricity transformer temperature data for fine-grained power grid monitoring \\
\midrule
Solar Power & 493,149 × 2 & Energy & High-frequency (1-minute) solar power generation data for photovoltaic system analysis \\
\midrule
ETTm2 & 69,680 × 8 & Energy & Secondary minute-resolution transformer dataset providing additional power infrastructure insights \\
\bottomrule
\end{tabular}
\end{table}

\section{Detailed Dataset Descriptions}

\subsection{Environmental Domain Datasets}

\textbf{Air Quality Dataset:} Contains hourly measurements of atmospheric pollutants and meteorological conditions collected from urban monitoring stations. Key variables include carbon monoxide (CO), benzene (C6H6), nitrogen oxides (NOx), and various sensor readings for pollution monitoring, alongside temperature, relative humidity, and absolute humidity measurements.

\textbf{WTH (Weather) Dataset:} Provides comprehensive meteorological observations including dry and wet bulb temperatures in both Fahrenheit and Celsius, dew point measurements, relative humidity, wind speed and direction, atmospheric pressure readings, and visibility conditions.

\textbf{River Flow Dataset:} Records daily streamflow measurements essential for hydrological modeling, flood prediction, and water resource management. The time series captures seasonal variations and extreme events in riverine systems.

\textbf{Extended Weather Dataset:} Features detailed atmospheric measurements including barometric pressure, potential temperature, vapor pressure components, specific humidity, water vapor concentration, air density, wind velocities, precipitation data, and solar radiation parameters.

\subsection{Energy Domain Datasets}

\textbf{Electricity Demand Dataset:} Captures electricity consumption patterns across five Australian states, providing insights into regional energy usage, demand forecasting, and grid management strategies.

\textbf{Wind Power Dataset:} High-resolution (1-minute interval) measurements of wind power generation, crucial for renewable energy integration, grid stability analysis, and short-term power forecasting applications.

\textbf{Solar Power Dataset:} Minute-level solar photovoltaic power output data enabling fine-grained analysis of solar energy patterns, cloud intermittency effects, and renewable energy variability studies.

\textbf{ETT (Electricity Transformer Temperature) Datasets:} Four complementary datasets (ETTh1, ETTh2, ETTm1, ETTm2) monitoring transformer temperatures at hourly and minute resolutions. These datasets are fundamental for power grid health monitoring, predictive maintenance, and electrical infrastructure management.

\textbf{Large-scale Electricity Dataset:} Encompasses consumption data from 321 individual consumers, providing a comprehensive view of distributed electricity usage patterns suitable for demand response analysis and consumer behavior modeling.

\textbf{TCPC (Temperature-Correlated Power Consumption):} Integrates environmental factors with power consumption across multiple zones, including temperature, humidity, wind speed, and diffuse radiation measurements alongside zonal energy usage data.

\textbf{Building Energy Dataset:} Detailed energy consumption monitoring of residential appliances and lighting systems, complemented by multi-zone temperature and humidity sensors, outdoor weather conditions, and building environmental parameters.

\subsection{Transportation Domain Datasets}

\textbf{Traffic Dataset:} Comprehensive highway traffic monitoring system covering 862 sensor locations, measuring vehicle occupancy rates and traffic flow patterns essential for intelligent transportation systems and congestion management.

\textbf{Metro Dataset:} Urban public transportation ridership data capturing passenger traffic volumes in metropolitan transit systems, valuable for public transportation planning and urban mobility analysis.

\subsection{Financial Domain Datasets}

\textbf{Exchange Rate Dataset:} Daily foreign exchange rate fluctuations for multiple international currency pairs, providing data for financial market analysis, currency risk assessment, and economic forecasting models.

\subsection{Healthcare Domain Datasets}

\textbf{National Illness Dataset:} Weekly surveillance data tracking influenza-like illness (ILI) prevalence across different age demographics and healthcare provider networks, supporting epidemiological research and public health monitoring.

\subsection{Astronomical Domain Datasets}

\textbf{Sunspot Dataset:} Long-term solar activity observations recording sunspot numbers, essential for space weather prediction, satellite operations planning, and understanding solar-terrestrial interactions.

\section{Dataset Statistics Summary}

The complete dataset collection comprises:
\begin{itemize}
    \item Total temporal observations: 1,758,768 time points
    \item Total data points: 31,479,451 (L × C)
    \item Temporal resolutions: 1-minute to weekly intervals
    \item Domain coverage: 7 distinct application areas
    \item Dimensionality range: 2 to 863 features per dataset
    \item Estimated storage requirement: 240.2 MB
\end{itemize}

The datasets provide extensive coverage across critical infrastructure sectors, environmental monitoring systems, and socio-economic indicators, making them suitable for comprehensive time series analysis, multivariate forecasting, and cross-domain pattern recognition research.

\section{Qualitative Analysis of the Learned Temporal Vocabulary}

To qualitatively understand the vocabulary discovered by our data-driven approach, Figure~\ref{fig:centers} visualizes the complete set of 32 cluster centroids, or ``temporal motifs'', learned from the dataset. The motifs are sorted in descending order of their frequency of occurrence (denoted by~$n$), providing a clear view into the structural composition of the ``language of time''.

\paragraph{A Hierarchy from Simple States to Complex Events.}
A striking feature revealed in Figure~\ref{fig:centers} is the emergent hierarchy of pattern complexity. The most frequent motifs, displayed in the top row, represent simple and fundamental states. For instance, Cluster~18 ($n=4452$) corresponds to a near-zero stable signal, while Cluster~21 ($n=2738$) and Cluster~1 ($n=2694$) represent high and medium constant values, respectively. These high-frequency patterns can be interpreted as the ``grammatical'' or functional components of the temporal language, akin to articles or prepositions in natural language, forming the stable background upon which more complex dynamics unfold.

Conversely, as we proceed to motifs with lower frequencies (middle and bottom rows), the patterns exhibit significantly greater complexity and convey more specific dynamic information. We can clearly identify distinct archetypes corresponding to fundamental temporal behaviors:
\begin{itemize}
    \item \textbf{Trends and Slopes:} Gentle upward (e.g., Cluster 2) and downward (e.g., Cluster 25) trends.
    \item \textbf{Troughs and Peaks:} U-shaped valleys (e.g., Cluster 10, 13) and bell-shaped crests (e.g., Cluster 28).
    \item \textbf{Sharp Transitions:} Rapid state changes, such as sharp rising edges (e.g., Cluster 16), S-shaped transitions (e.g., Cluster 17), and step-like functions (e.g., Cluster 22).
\end{itemize}
These rarer, more complex motifs act as the ``semantic'' core of the vocabulary, analogous to content-rich nouns and verbs that describe specific, meaningful events within the time series.

\paragraph{Qualitative Validation of the Linguistic Analogy.}
The structure of this learned vocabulary provides strong qualitative validation for our central hypothesis. The inverse relationship between pattern complexity and frequency---whereby simple, foundational patterns are ubiquitous and complex, event-specific patterns are rare---aligns perfectly with the quantitative findings of Zipf's Law presented in our earlier analysis. The ability to automatically discover such a rich, interpretable, and comprehensive lexicon from raw data demonstrates that complex time series dynamics are indeed compositional. This confirms that a finite set of reusable motifs forms the basis of observed signals, providing a solid foundation for treating time series analysis as a language modeling task.

\begin{figure}[htbp]
    \centering
    \includegraphics[width=\columnwidth]{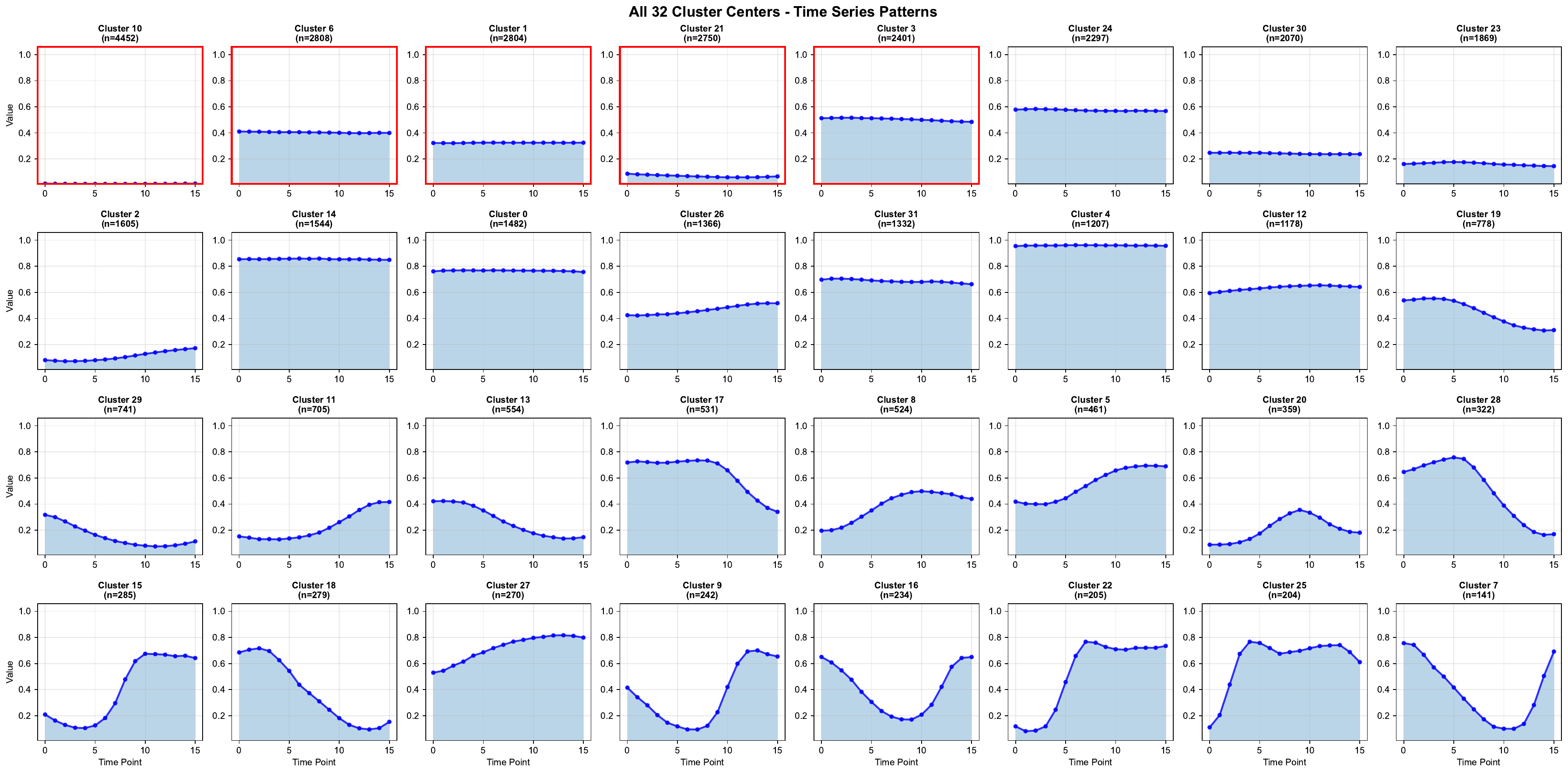}
    \caption{
        \textbf{The Learned Vocabulary of Temporal Motifs: Visualizing the 32 Cluster Centers.} 
        This figure displays the 32 cluster centers, or 'temporal motifs,' learned by the K-Means algorithm (K=32) from time series patches of length 16. Each plot represents a single prototypical pattern. The plots are sorted in descending order based on their frequency of occurrence (cluster size, denoted by~$n$), from the most common (Cluster 18, top-left) to the rarest (Cluster 7, bottom-right).
    }
    \label{fig:centers} 
\end{figure}

\end{document}